\documentclass[twoside]{article}

\usepackage[accepted]{aistats2021}

\usepackage[utf8]{inputenc} 
\usepackage[T1]{fontenc}    
\usepackage[english]{babel}

\usepackage{hyperref}       
\hypersetup{colorlinks,linkcolor={blue},citecolor={blue},urlcolor={blue}} 
\usepackage{url}            
\usepackage{booktabs}       
\usepackage{amsfonts}       
\usepackage{nicefrac}       
\usepackage{microtype}      

\usepackage{amsthm}
\usepackage{amssymb}
\usepackage{amsmath}
\usepackage{array,graphicx}

\usepackage[noend]{algpseudocode}

\usepackage{mathtools} 
\usepackage{stmaryrd}
\usepackage{bm}
\usepackage{tikz}
\usetikzlibrary{calc}
\usetikzlibrary{decorations}
\usepackage{multirow}
\usepackage{makecell}
\usepackage{thmtools, thm-restate}
\usepackage{subcaption}
\usepackage{float}
\usepackage{wrapfig}

\usepackage{tikz}
\usepackage{pgfplots}
\usetikzlibrary{arrows}
\usetikzlibrary{positioning}
\usetikzlibrary{decorations.text}
\usetikzlibrary{decorations.markings}
\usetikzlibrary{decorations.shapes}

\usepackage{lscape}

\definecolor{myteal}{HTML}{2380D8}
\definecolor{myred}{HTML}{EA2222}
\definecolor{mypurple}{HTML}{9900CC}
\definecolor{mygrey}{HTML}{494949}
\definecolor{mygreen}{HTML}{008000}
\definecolor{myblack}{HTML}{000000}
\definecolor{mywhite}{HTML}{FFFFFF}
\definecolor{putblue}{RGB}{0,96,124}
\definecolor{putred}{RGB}{204,33,69}

\usepackage{adjustbox}

\makeatletter

\newcommand*{\eifstartswith}{\@expandtwoargs\ifstartswith}
\newcommand*{\ifstartswith}[2]{%
  \if\@car#1.\@nil\@car#2.\@nil
    \expandafter\@firstoftwo
  \else
    \expandafter\@secondoftwo
  \fi}
\theoremstyle{definition}
\newtheorem{definition}{Definition}

\newtheorem{lemma}{Lemma}

\newcommand{\Algo}[1]{\textsc{#1}}

\renewcommand{\vec}[1]{\boldsymbol{#1}}

\newcommand{\bx}{\vec{x}}
\newcommand{\by}{\vec{y}}
\newcommand{\bz}{\vec{z}}

\newcommand{\calX}{\mathcal{X}}
\newcommand{\calY}{\mathcal{Y}}

\newcommand{\calD}{\mathcal{D}}
\newcommand{\calQ}{\mathcal{Q}}
\newcommand{\calL}{\mathcal{L}}
\newcommand{\calH}{\mathcal{H}}
\newcommand{\calS}{\mathcal{S}}

\newcommand{\dataset}{{\cal D}}

\newcommand{\heta}{\hat{\eta}}
\newcommand{\hy}{\hat{y}}

\newcommand\R{\mathbb{R}}   

\newcommand{\pa}[1]{\mathrm{pa}(#1)}

\newcommand{\Path}[1]{\mathrm{Path}(#1)}
\newcommand{\lenpath}{\mathrm{len}}

\newcommand{\prob}{\mathbf{P}}

\newcommand{\childs}[1]{\mathrm{Ch}(#1)}

\renewcommand{\root}{r}
\newcommand{\nodes}{V}
\newcommand{\vertices}{V}
\newcommand{\tree}{T}
\newcommand{\leaves}{L}
\newcommand{\leafnode}{l}
\newcommand{\labels}{\calL}
\newcommand{\vertex}{v}
\newcommand{\node}{v}

\newcommand{\multilabel}{multi-label}

\algnewcommand{\IIf}[1]{\State\algorithmicif\ #1\ \algorithmicthen}
\algnewcommand{\IElse}[1]{\State\algorithmicelse\ #1\ }
\algnewcommand{\IElseIf}[1]{\State\algorithmicelse \algorithmicif\ #1\  \algorithmicthen}
\algnewcommand{\EndIIf}{\unskip\ \algorithmicend\ \algorithmicif}
\algnewcommand{\IfThen}[2]{\State\algorithmicif\ #1\ \algorithmicthen\ #2\ }
\algnewcommand{\ForDo}[2]{\State\algorithmicfor\ #1\ \algorithmicdo\ #2\ }

\newcommand{\assert}[1]{\llbracket #1 \rrbracket}

\newcommand{\given}{\, | \,}

\DeclareMathOperator*{\argmax}{\arg \max}

\newcommand{\CommentSS}[1]{{\scriptsize \Comment{#1}}}

\usepackage{algorithm}
\usepackage{algorithmic}

\usepackage[round]{natbib}

\renewcommand{\refname}{References}

\begin{document}

\runningtitle{Online probabilistic label trees}

%
\runningauthor{K. Jasinska-Kobus, M. Wydmuch, D. Thiruvenkatachari, K. Dembczy\'nski}


\twocolumn[
\aistatstitle{Online probabilistic label trees}
\aistatsauthor{Kalina Jasinska-Kobus$^*$ \\
ML Research at Allegro.pl, Poland \\
Poznan University of Technology, Poland \\
\nolinkurl{kjasinska@cs.put.poznan.pl}
\And
Marek Wydmuch$^*$ \\
Poznan University of Technology, Poland \\
\nolinkurl{mwydmuch@cs.put.poznan.pl}
\AND
Devanathan Thiruvenkatachari \\
Yahoo Research, New York, USA \\
\And
Krzysztof Dembczy\'nski \\
Yahoo Research, New York, USA \\
Poznan University of Technology, Poland \\

}
\aistatsaddress{\small$^*$ Equal contribution}
]


\begin{abstract}
We introduce \emph{online probabilistic label trees} (OPLTs),
an algorithm that trains a label tree classifier in a fully online manner
without any prior knowledge about the number of training instances, their features and labels.
\Algo{OPLT}s are characterized by low time and space complexity as well as strong theoretical guarantees.
They can be used for online multi-label and multi-class classification,
including the very challenging scenarios of one- or few-shot learning.
We demonstrate the attractiveness of \Algo{OPLT}s in a wide empirical study on several instances of the tasks mentioned above.
\end{abstract}


\section{Introduction}
\label{sec:introduction}


In modern machine learning applications, the label space can be enormous, containing even millions of different labels.
Problems of such scale are often referred to as \emph{extreme classification}.
Some notable examples of such problems are tagging of text documents~\citep{Dekel_Shamir_2010}, content annotation for multimedia search~\citep{Deng_et_al_2011}, and different types of recommendation, including webpages-to-ads~\citep{Beygelzimer_et_al_2009b}, ads-to-bid-words~\citep{Agrawal_et_al_2013,Prabhu_Varma_2014}, users-to-items~\citep{Weston_et_al_2013, Zhuo_et_al_2020}, queries-to-items~\citep{Medini_et_al_2019}, or items-to-queries~\citep{Chang_et_al_2020}.
In these practical applications, learning algorithms run in rapidly changing environments.
Hence, the space of labels and features might grow over time as new data points arrive. 
Retraining the model from scratch every time a new label is observed is computationally expensive, requires storing all previous data points, and introduces long retention before the model can predict new labels.
Therefore, it is desirable for algorithms operating in such a setting to work in an incremental fashion,
efficiently adapting to the growing label and feature space.

To tackle extreme classification problems efficiently, we consider a class of label tree algorithms that use a hierarchical structure of classifiers to reduce the computational complexity of training and prediction.
The tree nodes contain classifiers that direct the test examples from the root down to the leaf nodes, where each leaf corresponds to one label.
We focus on a subclass of label tree algorithms that uses probabilistic classifiers.
Examples of such algorithms for multi-class classification include hierarchical softmax (\Algo{HSM})~\citep{Morin_Bengio_2005},
implemented for example in \Algo{fastText}~\citep{Joulin_et_al_2016},
and conditional probability estimation tree (CPET)~\citep{Beygelzimer_et_al_2009b}.
The generalization of this idea to multi-label classification is known under the name of probabilistic label trees (\Algo{PLT}s)~\citep{Jasinska_et_al_2016}, and has been recently implemented in several state-of-the-art algorithms: \Algo{Parabel}~\citep{Prabhu_et_al_2018}, \Algo{extremeText}~\citep{Wydmuch_et_al_2018},
\Algo{Bonsai Tree}~\citep{Khandagale_et_al_2019},
and \Algo{AttentionXML}~\citep{You_et_al_2019}.
While some of the above algorithms use incremental procedures to train node classifiers,
only CPET allows for extending the model with new labels,
but it only works for multi-class classification.
For all the other algorithms, a label tree needs to be given before training of the node classifiers.

In this paper, we introduce \emph{online probabilistic label trees} (OPLTs), an algorithm for multi-class and multi-label problems,
which trains a label tree classifier in a \emph{fully} online manner.
This means that the algorithm does not require any prior knowledge about the number of training instances, their features and labels.
The tree is updated every time a new label arrives with a new example, in a similar manner as in CPET~\citep{Beygelzimer_et_al_2009b}, but the mechanism used there has been generalized to multi-label data.
Also, new features are added when they are observed. This can be achieved by feature hashing~\citep{Weinberger_et_al_2009} as in the popular Vowpal Wabbit package~\citep{Langford_et_al_2007}.
We rely, however, on a different technique based on recent advances in the implementation of hash maps, namely the Robin Hood hashing~\citep{Celis_et_al_1985}.

We require the model trained by \Algo{OPLT} to be equivalent to a model trained as a tree structure would be known from the very beginning.
In other words, the node classifiers should be exactly the same as the ones trained on the same sequence of training data using the same incremental learning algorithm, but with the tree produced by \Algo{OPLT} given as an input parameter before training them.
We refer to an algorithm satisfying this requirement as a \emph{proper} online \Algo{PLT}.
If the incremental tree can be built efficiently, then we additionally say that the algorithm is also \emph{efficient}.
These properties are important as a proper and efficient online \Algo{PLT} algorithm possesses similar guarantees as \Algo{PLT}s in terms of computational complexity~\citep{Busa-Fekete_et_al_2019} and statistical performance~\citep{Wydmuch_et_al_2018}.

To our best knowledge, the only algorithm that also addresses the problem of fully online learning in the extreme multi-class and multi-label setting is the recently introduced contextual memory tree (\Algo{CMT})~\citep{Sun_et_al_2019}, which is a specific online key-value structure that can be applied to a wide spectrum of online problems.
More precisely, \Algo{CMT} stores observed examples in the near-balanced binary tree structure that grows with each new example. The problem of mapping keys to values is converted into a collection of classification problems in the tree nodes, which predict which sub-tree contains the best value corresponding to the key.
\Algo{CMT} has been empirically proven to be useful for the few-shot learning setting in extreme multi-class classification, where it has been used directly as a classifier, and for extreme multi-label classification problems, where it has been used to augment an online one-versus-rest (OVR) algorithm.
In the experimental study, we compare \Algo{OPLT} with its offline counterparts and \Algo{CMT} on both extreme multi-label classification and few-shot multi-class classification tasks.

Some other existing extreme classification approaches can be tried to be used in the fully online setting,
but the adaptation is not straight-forward and there does not exist any such algorithm.
For example, the efficient OVR approaches (e.g., \Algo{DiSMEC}~\citep{Babbar_Scholkopf_2017},
\Algo{PPDSparse}~\citep{Yen_et_al_2017}, \Algo{ProXML}~\citep{Babbar_Scholkopf_2019}) work only in the batch mode.
Interestingly, one way of obtaining a fully online OVR is to use \Algo{OPLT} with a 1-level tree.
Popular decision-tree-based approaches, such as \Algo{FastXML}~\citep{Prabhu_Varma_2014}, also work in the batch mode.
An exception is \Algo{LOMtree}~\citep{Choromanska_Langford_2015}), which is an online algorithm.
It can be adapted to the fully online setting,
but as shown in~\citep{Sun_et_al_2019} its performance is worse than the one of \Algo{CMT}.
Recently, the idea of soft trees, closely related to hierarchical mixture of experts~\citep{Jordan_Jacobs_1994},
has gained increasing attention in deep learning community~\citep{Frosst_Hinton_2017, Kontschieder_et_al_2015, Hehn_et_al_2020}.
However, it has been used neither in the extreme nor in the fully online setting.


The paper is organized as follows. 
In Section~\ref{sec:xmlc} we define the problem of extreme multi-label classification (XMLC).
Section~\ref{sec:plt} recalls the \Algo{PLT} model.
Section~\ref{sec:oplt} introduces the \Algo{OPLT} algorithm, defines the desired properties
and shows that the introduced algorithm satisfies them.
Section~\ref{sec:experiments} presents experimental results.
The last section concludes the paper.

\section{Extreme multi-label classification}
\label{sec:xmlc}

Let $\calX$ denote an instance space and $\labels$
be a finite set of $m$ labels.
We assume that an instance $\bx \in \calX$ is associated with a subset of labels $\labels_{\bx} \subseteq \calL$
(the subset can be empty);
this subset is often called the set of relevant or positive labels,
while the complement $\labels \backslash \labels_{\bx}$ is considered as irrelevant or negative for $\bx$.
We identify the set $\labels_{\bx}$ of relevant labels with the binary vector $\by = (y_1,y_2, \ldots, y_m)$,
in which $y_j = 1 \Leftrightarrow j \in \labels_{\bx}$.
By $\calY = \{0, 1\}^m$ we denote the set of all possible label vectors.
We assume that observations $(\bx, \by)$ are generated independently and identically
according to a probability distribution $\prob(\bx, \by)$ defined on $\calX \times \calY$.
Notice that the above definition of multi-label classification includes multi-class classification
as a special case in which $\|\by\|_1=1$ ($\|\cdot\|$ denotes a vector norm).
In case of XMLC, we assume $m$ to be a large number 
but the size of the set of relevant labels $\labels_{\bx}$ is usually much smaller than $m$,
that is $|\labels_{\bx}| \ll m$.


\begin{algorithm*}[ht!]
\caption{\footnotesize \Algo{IPLT.Train}$(T, A_{\textrm{online}}, \calD)$}
\label{alg:plt-incremental-learning}
\begin{algorithmic}[1]
\footnotesize
\State ${H_T} = \emptyset$ \CommentSS{Initialize a set of node probabilistic classifiers}
\ForDo{$\vertex \in \vertices_{T}$}{$\heta(v) = \textsc{NewClassifier}()$, $H_T = H_T \cup \{\heta(v)\}$} \CommentSS{Initialize binary classifier for each node in the tree.}
\For{$i = 1 \to n$}  \CommentSS{For each observation in the training sequence}
    \State $(P, N) = \mathrm{\textsc{AssignToNodes}}(T, \bx_i, \calL_{\bx_i})$ \CommentSS{Compute its positive and negative nodes}
    \ForDo{$v \in P$}
        {$A_{\textrm{online}}\textsc{.Update}(\heta(v), (\bx_i, 1))$} \CommentSS{Update all positive nodes with a positive update with $\bx$.}
    \ForDo{$v \in N$}{$A_{\textrm{online}}\textsc{.Update}(\heta(v), (\bx_i, 0))$} \CommentSS{Update all negative nodes with a negative update with $\bx$.}
    \EndFor
\State \textbf{return} $H_T$ \CommentSS{Return the set of node probabilistic classifiers}
\end{algorithmic}
\end{algorithm*}%
\vspace{-5pt}


\section{Probabilistic label trees}
\label{sec:plt}

We recall the definition of probabilistic label trees (\Algo{PLT}s), introduced in~\citep{Jasinska_et_al_2016}.
\Algo{PLT}s follow a label-tree approach to efficiently solve the problem of estimation of marginal probabilities of labels in multi-label problems.
They reduce the original problem to a set of binary problems organized in the form of a rooted, leaf-labeled tree with $m$ leaves.
We denote a single tree by $\tree$, a root node by $\root_T$, and the set of leaves by $\leaves_T$.
The leaf $\leafnode_j \in \leaves_T$ corresponds to the label $j \in \labels$.
The set of leaves of a (sub)tree rooted in node $v$ is denoted by $\leaves_v$.
{The set of labels corresponding to all leaf nodes in $\leaves_v$ is denoted by $\calL_{v}$}.
The parent node of $v$ is denoted by $\pa{\node}$, and the set of child nodes by $\childs{\node}$.
The path from node $v$ to the root is denoted by $\Path{\node}$.
The length of the path is the number of nodes on the path, which is denoted by $\lenpath_\node$.
The set of all nodes is denoted by $\nodes_T$.
The degree of a node $\node \in \nodes_T$,  being the number of its children, is denoted by $\deg_\node=|\childs{\node}|$.

\Algo{PLT} uses tree $T$ to factorize conditional probabilities of labels,
$\eta_j(\bx) = \prob(y_j = 1 \vert \bx) = \prob(j \in \calL_{\bx} \vert \bx)$.
To this end let us define for every $\by$ a corresponding vector $\bz$ of length $|V_T|$,%
whose coordinates, indexed by $v \in V_T$, are given by:
\begin{equation}
z_v = \assert{\textstyle \sum_{\ell_j \in L_v} y_{j} \ge 1} \,.
\label{eqn:z}
\end{equation}
In other words, the element $z_v$ of $\bz$, corresponding to the node $v \in V_{\tree}$, is set to one iff $\by$ contains at least one label corresponding to leaves in $L_v$.
With the above definition, it holds for any {node} $v \in V_T$ that:
\begin{equation}
\eta_v(\bx) = \prob(z_v = 1 \given \bx) =  \prod_{v' \in \Path{v}} \eta(\bx, v') \,,
\label{eqn:plt_factorization}
\end{equation}
where $\eta(\bx, v) = \prob(z_v = 1 \vert z_{\pa{v}} = 1, \bx)$ for non-root nodes,
and $\eta(\bx, v) = \prob(z_v = 1 \given \bx)$ for the root~(see, e.g.,~\citet{Jasinska_et_al_2016}).
Notice that for leaf nodes we get the conditional probabilities of labels, i.e.,
\begin{equation}
\eta_{\leafnode_j}(\bx) = \eta_j(\bx) \,, \quad \textrm{for~} l_j \in L_T \,.
\label{eqn:plt_leaf_prob}
\end{equation}

For a given $T$ it suffices to estimate $\eta(\bx, v)$, for $v \in V_T$, to build a \Algo{PLT} model.
To this end one usually uses a function class $\calH : \R^d \mapsto [0,1]$ which %
contains probabilistic classifiers of choice, for example,
logistic regression.
We assign a classifier from $\calH$ to each node of the tree $T$. We index this %
set of classifiers by the elements of $V_T$ as $H = \{ \heta(v) \in \calH : v \in V_T \}$.
Training is performed usually on a dataset $\dataset = \{ (\bx_{i},\by_{i})\}_{i=1}^n$ consisting of $n$ tuples of feature vector $\bx_i \in \R^d$ and label vector $\by_i\in \{ 0,1\}^m$.
Because of factorization (\ref{eqn:plt_factorization}),
node classifiers can be trained as independent tasks.

The quality of the estimates $\heta_j(\bx)$, $j \in \calL$, can be expressed in terms of the $L_1$-estimation error in each node classifier, i.e., by $\left | \eta(\bx,v) - \heta(\bx, v) \right |$. \Algo{PLT}s obey the following bound~\citep{Wydmuch_et_al_2018}.
\begin{restatable}{theorem}{estimationregret}
\label{thm:estimation_regret}
For any tree $T$ and $\prob(\by \vert \bx)$ the following holds for $v \in V_T$:
$$
\left | \eta_j(\bx) - \heta_j(\bx) \right |
\leq \!\!\!\!\! \sum_{v' \in \Path{l_j}} \!\!\!\!\! \eta_{\pa{v'}}(\bx) \left | \eta(\bx,v')\!-\! \heta(\bx, v')  \right | \,,
$$
where for the root node $\eta_{\pa{r_T}}(\bx) = 1$.
\end{restatable}

Prediction for a test example $\bx$ relies on searching the tree.
For metrics such as precision@$k$, the optimal strategy is to predict $k$ labels with the highest marginal probability $\eta_j(\bx)$.
To this end, the prediction procedure traverses the tree using the uniform-cost search
to obtain the top $k$ estimates $\heta_j(\bx)$ (see Appendix~\ref{app:prediction} for the pseudocode).


\begin{algorithm*}[ht!]
\caption{\footnotesize \Algo{OPLT.Init}$()$}
\label{alg:oplt-init}
\begin{algorithmic}[1]
\footnotesize
\State $r_T = \textsc{NewNode()}$, $V_T = \{r_T\}$ \CommentSS{Create the root of the tree}
\State {$\heta(r_T) = \textsc{NewClassifier}()$, $H_T = \{\heta_T(r_T)\}$} \CommentSS{Initialize a new classifier in the root}
\State {$\hat{\theta}(r_T) = \textsc{NewClassifier}()$, $\Theta_T = \{\theta(r_T)\}$} \CommentSS{Initialize an auxiliary classifier in the root}
\end{algorithmic}
\end{algorithm*}

\begin{algorithm*}[ht!]
\caption{\footnotesize \Algo{OPLT.Train}$(\mathcal{S}, A_{\textrm{online}}, A_{\textrm{policy}})$}
\label{alg:oplt}
\begin{algorithmic}[1]
\footnotesize

\For{$(\bx_t, \calL_{\bx_t}) \in \calS$}  \CommentSS{For each observation in $\mathcal{S}$}
    \IfThen{ $\calL_{\bx_t} \setminus \calL_{t-1} \neq \emptyset$}{\textsc{UpdateTree}$(\bx_t, \calL_{\bx_t}, A_{\textrm{policy}})$} \CommentSS{If the obser. contains new labels, add them to the tree}
    \State \textsc{UpdateClassifiers}$(\bx_t, \calL_{\bx_t}, A_{\textrm{online}})$ \CommentSS {Update the classifiers}
    \State \textbf{send} $H_t, T_t = H_T, V_T$  \CommentSS{Send the node classifiers and the tree structure.}
\EndFor
\end{algorithmic}
\end{algorithm*}
\vspace{-5pt}

\begin{algorithm*}[ht]
\caption{\footnotesize \Algo{OPLT.UpdateTree}$(\bx, \calL_{\bx},  A_{\textrm{policy}})$}
\label{alg:oplt-update_tree}
\begin{algorithmic}[1]
\footnotesize
\For{ $j \in \labels_{\bx} \setminus \calL_{t-1}$}\CommentSS{For each new label in the observation}
    \If{$\calL_T$ is $\emptyset$}
    $\textsc{label}(r_T) = j$ \CommentSS{If no labels have been seen so far, assign label $j$ to the root node}
    \Else \CommentSS{If there are already labels in the tree.}
        \State {$v,\ {insert} = A_{\textrm{policy}}(\bx, j, \calL_{\bx})$} \CommentSS{Select a variant of extending the tree}
        \IfThen{${insert}$}{$\textsc{InsertNode}(v)$} \CommentSS{Insert an additional node if needed.}
        \State{$\textsc{AddLeaf}(j, v)$}  \CommentSS{Add a new leaf for label $j$.}
    \EndIf
\EndFor
\end{algorithmic}
\end{algorithm*}

\begin{algorithm*}[t]
\caption{\footnotesize \Algo{OPLT.InsertNode}$(v)$}
\label{alg:oplt-insert_node}
\begin{algorithmic}[1]
\footnotesize
    \State{$v' = \textsc{NewNode}$(), $V_T = V_T \cup \{v'\}$} \CommentSS{Create a new node and add it to the tree nodes}
    \If{$\textsc{IsLeaf}(v)$}
        $\textsc{Label}(v') = \textsc{Label}(v)$, $\textsc{Label}(v) = \textsc{Null}$ \CommentSS{If node $v$ is a leaf reassign label of $v$ to $v'$}
    \Else \CommentSS{Otherwise}
        \State{$\childs{v'} = \childs{v}$}  \CommentSS{All children of $v$ become children  of $v'$}
        \ForDo{$v_{\textrm{ch}} \in \childs{v'}$}{$\pa{v_{\textrm{ch}}} = v'$} \CommentSS{And $v'$ becomes their parent}
    \EndIf
    \State{$\childs{v} = \{v'\}$, $\pa{v'} = v$} \CommentSS{The new node $v'$ becomes the only child of $v$}
    \State{$\heta(v') = \textsc{Copy}(\hat\theta(v))$, $H_T = H_T \cup \{\heta(v')\}$} \CommentSS{Create a classifier.}
    \State{$\hat{\theta}(v') = \textsc{Copy}(\hat\theta(v))$, $\Theta_T = \Theta_T \cup \{\hat\theta(v')\}$} \CommentSS{And an auxiliary classifier.}
\end{algorithmic}
\end{algorithm*}
\vspace{-5pt}


\section{Online probabilistic label trees}
\label{sec:oplt}


A \Algo{PLT} model can be trained incrementally,
on observations from $\calD = \{(\bx_i, \by_i) \}_{i=1}^n$,
using an incremental learning algorithm $A_{\textrm{online}}$ for updating the tree nodes.
Such \emph{incremental} \Algo{PLT} (\Algo{IPLT}) is given in Algorithm~\ref{alg:plt-incremental-learning}.
In each iteration, it first identifies the set of \emph{positive} and \emph{negative nodes} using the \Algo{AssignToNodes} procedure
(see Appendix~\ref{app:training} for the pseudocode and description).
The positive nodes are those for which the current training example is treated as positive (i.e, $(\bx, z_v = 1)$),
while the negative nodes are those for which the example is treated as negative (i.e., $(\bx, z_v = 0)$).
Next, \Algo{IPLT} appropriately updates classifiers in the identified nodes.
Unfortunately, the incremental training in \Algo{IPLT} requires the tree structure $T$ to be given in advance.

To construct a tree, at least the number $m$ of labels needs to be known.
More advanced tree construction procedures exploit additional information like feature values or label co-occurrence~\citep{Prabhu_et_al_2018, Khandagale_et_al_2019}.
In all such algorithms, the tree is built prior to the learning of node classifiers.
Here, we analyze a different scenario in which
an algorithm operates on a possibly infinite sequence of training instances,
and the tree is constructed online, simultaneously with incremental training of node classifiers,
without any prior knowledge of the set of labels or training data.

Let us denote a sequence of observations by $\calS = \{(\bx_i, \calL_{\bx_i})\}_{i=1}^\infty$
and a subsequence consisting of the first $t$ instances by $\calS_t$.
We use here $\calL_{\bx_i}$ instead of $\by_i$
as the number of labels $m$, which is also the length of $\by_i$, increases over time in this online scenario.%
\footnote{The same applies to $\bx_i$ as the number of features also increases.
However, we keep the vector notation in this case, as it does not impact the algorithm's description.}
Furthermore, let the set of labels observed in $\calS_t$ be denoted by $\calL_t$,
with $\calL_0 = \emptyset$.
An online algorithm returns at step $t$ a tree structure $T_t$
constructed over labels in $\calL_t$ and a set of node classifiers $H_t$.
Notice that the tree structure and the set of classifiers change in each iteration
in which one or more new labels are observed.
Below we discuss two properties that are desired for such online algorithms,
defined in relation to the \Algo{IPLT} algorithm given above.
\begin{definition}[A proper online \Algo{PLT} algorithm]
\label{oplt-proper}
Let $T_t$ and $H_t$ be respectively a tree structure and a set of node classifiers
trained on a sequence $\calS_t$ using an online algorithm $B$.
We say that $B$ is a \emph{proper online \Algo{PLT} algorithm},
when for any $\calS$ and $t$ we have that
\begin{itemize}
\item $l_j \in L_{T_t}$ iff $j \in \labels_t$, i.e.,
leaves of $T_t$ correspond to all labels observed in $S_t$,
\item and $H_t$ is exactly the same as $H = \textsc{IPLT.Train}(T_t, A_{\textrm{online}}, \mathcal{S}_t)$,
i.e., node classifiers from $H_t$ are the same
as the ones trained incrementally by Algorithm~\ref{alg:plt-incremental-learning}
on $\calD = \calS_t$ and tree $T_t$ given as input parameter.
\end{itemize}
\end{definition}
\noindent
In other words, we require that whatever tree the online algorithm produces,
the node classifiers should be trained in the same way
as if the tree was known from the very beginning of training.
Thanks to that, we can control the quality of each node classifier,
as we are not missing any update.
Since the model produced by a proper online \Algo{PLT} is the same as of \Algo{IPLT},
the same statistical guarantees apply to both of them.


The above definition can be satisfied by a naive algorithm that stores all observations seen so far, uses them in each iteration to build a tree and train node classifiers with the \Algo{IPLT} algorithm from scratch. This approach is costly. Therefore, we also demand an online algorithm to be space and time-efficient in the following sense.
\begin{definition}[An efficient online \Algo{PLT} algorithm]
\label{oplt-efficient}
Let $T_t$ and $H_t$ be respectively a tree structure and a set of node classifiers
trained on a sequence $\calS_t$ using an online algorithm $B$.
Let $C_s$ and $C_t$ be the space and time training cost of \Algo{IPLT}
trained on sequence $\calS_t$ and tree $T_t$.
An online algorithm is an \emph{efficient online \Algo{PLT} algorithm}
when for any $S$ and $t$ we have its space and time complexity to be in constant factor of $C_s$ and $C_t$, respectively.
\end{definition}
In this definition, we abstract from the actual implementation of \Algo{IPLT}.
In other words, the complexity of an efficient online \Algo{PLT} algorithm
depends directly on design choices for~\Algo{IPLT}.
The space complexity is upperbounded by $2m-1$ (i.e., the maximum number of node models),
but it also depends on the chosen type of node models and the way of storing them.
Let us also notice that the definition implies
that the update of a tree structure has to be in a constant factor of the training cost of a single instance. 

\subsection{Online tree building and training of node classifiers}
\label{sec:oplt-build-and-train}

Below we describe an online \Algo{PLT} algorithm that,
as we show in subsection~\ref{sec:oplt-theory},
satisfies both properties defined above.
It is similar to \Algo{CPET}~\citep{Beygelzimer_et_al_2009b},
but extends it to \multilabel{} problems
and trees of any shape. 
We refer to this algorithm as \Algo{OPLT}.

The pseudocode is presented in Algorithms~\ref{alg:oplt-init}-\ref{alg:oplt-update_classifiers}.
In a nutshell, \Algo{OPLT} proceeds observations from $\calS$ sequentially,
updating node classifiers.
For new incoming labels, it creates new nodes according to a chosen tree building policy
which is responsible for the main logic of the algorithm.
Each new node $v$ is associated with two classifiers,
a regular one $\heta(v) \in H_T$,
and an \emph{auxiliary} one $\hat\theta(v) \in \Theta_T$,
where $H_T$ and $\Theta_T$ denote the corresponding sets of node classifiers.
The task of the auxiliary classifiers is to accumulate positive updates.
The algorithm uses them later to initialize classifiers associated with new nodes added to a tree.
They can be removed if a given node will not be used anymore to extend the tree.
A particular criterion for removing an auxiliary classifier depends, however, on a tree building policy.

The algorithm starts with \Algo{OPLT.Init} procedure, presented in Algorithm~\ref{alg:oplt-init}, that initializes a tree with only a root node $v_{r_T}$ and corresponding classifiers, $\heta(v_{r_T})$ and $\hat\theta(v_{r_T})$.
Notice that the root has both classifiers initialized from the very beginning without a label assigned to it.
Thanks to this, the algorithm can properly estimate the probability of $\prob(\by = \vec{0} \given \bx)$.
From now on, \Algo{OPLT.Train}, outlined in Algorithm~\ref{alg:oplt}, administrates the entire process. In its main loop, observations from $\calS$ are proceeded sequentially. If a new observation contains one or more new labels
then the tree structure is appropriately extended by calling \Algo{UpdateTree}.
The node classifiers are updated in \Algo{UpdateClassifiers}.
After each iteration $t$, the algorithm sends $H_T$ along with the tree structure $T$,
respectively as $H_t$ and $T_t$, to be used outside the algorithm for prediction tasks.
We assume that tree $T$ along with sets of its all nodes $V_T$ and leaves $L_T$,
as well as sets of classifiers $H_T$ and $\Theta_T$,
are accessible to all the algorithms discussed below.

Algorithm~\ref{alg:oplt-update_tree}, \Algo{UpdateTree}, builds the tree structure.
It iterates over all new labels from $\calL_{\bx}$.
If there were no labels in the sequence $\calS$ before,
the first new label taken from $\calL_{\bx}$ is assigned to the root node.
Otherwise, the tree needs to be extended by one or two nodes according to a selected tree building policy.
One of these nodes is a leaf to which the new label will be assigned.
There are in general three variants of performing this step
illustrated in Figure~\ref{fig:oplt-tree_building}.
The first one relies on selecting an internal node $v$
whose number of children is lower than the accepted maximum,
and adding to it a child node $v''$ with the new label assigned to it.
In the second one, two new child nodes, $v'$ and $v''$, are added to a selected internal node $v$.
Node $v'$ becomes a new parent of child nodes of the selected node $v$,
i.e., the subtree of $v$ is moved down by one level.
Node $v''$ is a leaf with the new label assigned to it.
The third variant is a modification of the second one.
The difference is that the selected node $v$ is a leaf node.
Therefore there are no children nodes to be moved to $v'$,
but label of $v$ is reassigned to $v'$.
The $A_{\textrm{policy}}$ method encodes the tree building policy,
i.e., it decides which of the three variants to follow and selects the node $v$.
The additional node $v'$ is inserted by the \Algo{InsertNode} method.
Finally, a leaf node is added by the \Algo{AddLeaf} method.
We discuss the three methods in more detail below.

\begin{algorithm*}[t]
\caption{\footnotesize \Algo{OPLT.AddLeaf}$(j, v)$}
\label{alg:oplt-add_leaf}
\begin{algorithmic}[1]
\footnotesize
    \State $v'' = \textsc{NewNode()}$, $V_T = V_T \cup \{v''\}$ \CommentSS{Create a new node and add it to the tree nodes}
    \State $\childs{v} = \childs{v} \cup \{v''\}$,  $\pa{v''} = v$, $\textsc{label}(v'') = j$ \CommentSS{Add this node to children of $v$ and assign label $j$ to the node $v''$}
    \State {$\heta(v'') = \textsc{InverseClassifier}(\hat\theta(v))$, $H_T = H_T \cup \{ \heta(v'')\} $} \CommentSS{Initialize a classifier for $v''$}
    \State{$\hat{\theta}(v'') = \textsc{NewClassifier}()$, $\Theta_T = \Theta_T \cup \{ \hat{\theta}(v'')\}$} \CommentSS{Initialize an auxiliary classifier for $v''$}
\end{algorithmic}
\end{algorithm*}

\begin{figure*}[h]
\begin{tabular}{p{0.225\textwidth}p{0.225\textwidth}p{0.225\textwidth}p{0.225\textwidth}}
\scalebox{0.8}{
\begin{tikzpicture}[scale = 1,every node/.style={scale=1},
		regnode/.style={circle,draw,minimum width=1.5ex,inner sep=0pt},
		leaf/.style={circle,fill=black,draw,minimum width=1.5ex,inner sep=0pt},
		pleaf/.style={rectangle,rounded corners=1ex,draw,font=\scriptsize,inner sep=3pt},
		pnode/.style={rectangle,rounded corners=1ex,draw,font=\scriptsize,inner sep=3pt},
		rootnode/.style={rectangle,rounded corners=1ex,draw,font=\scriptsize,inner sep=3pt},
		level/.style={sibling distance=6em/#1, level distance=4ex}
	]
	\node (z) [rootnode,] {\color{white}{$v$}}
    child {node [fill=white!10!white!90] (a) [pnode] {$v_1$}
    		child {node [fill=white!10!white!90, label=below:{\makebox[1cm][c]{$y_1$}}] (a1) [pnode] {$v_2$}
    		edge from parent node[above left]{}
    }
    child {node [label=below:{\makebox[1cm][c]{$y_2$}},fill=white!10!white!90] (g) [pleaf] {$v_3$} edge from parent node[above]{}}
    		edge from parent node[above left]{}
    }
	child {node (j) [pnode] {\color{white}{$v$}}
		child {node [label=below:{\makebox[1cm][c]{$y_2$}}] (k) [pleaf] {\color{white}{$v$}} edge from parent node[above left]{}}
		child {node [label=below:{\makebox[1cm][c]{$y_3$}}] (l) [pleaf] {\color{white}{$v$}}
			{
				child [grow=right] {node (s) {} edge from parent[draw=none]
					child [grow=up] {node (t) {} edge from parent[draw=none]
						child [grow=up] {node (u) {} edge from parent[draw=none]}
					}
				}
			}
			edge from parent node[above right]{}
		}
		edge from parent node[above right]{}
	};
\end{tikzpicture}}
&
\scalebox{0.8}{
\begin{tikzpicture}[scale = 1,every node/.style={scale=1},
    		regnode/.style={circle,draw,minimum width=1.5ex,inner sep=0pt},
    		leaf/.style={circle,fill=black,draw,minimum width=1.5ex,inner sep=0pt},
    		pleaf/.style={rectangle,rounded corners=1ex,draw,font=\scriptsize,inner sep=3pt},
    		pnode/.style={rectangle,rounded corners=1ex,draw,font=\scriptsize,inner sep=3pt},
    		rootnode/.style={rectangle,rounded corners=1ex,draw,font=\scriptsize,inner sep=3pt},
    		level/.style={sibling distance=6em/#1, level distance=4ex},
    		level 3/.style={sibling distance=3em, level distance=4ex}
    	]
    	\node (z) [rootnode] {\color{white}{r}}
    	child {node [fill=black!30!white!70] (a) [pnode] {$v_1$}
    		child {node [fill=white!10!white!90, label=below:{\makebox[1cm][c]{$y_1$}}] (a1) [pleaf] {$v_2$}
    		edge from parent node[above left]{}
    	    }
    	    child {node [fill=white!10!white!90,
    	    label=below:{\makebox[1cm][c]{$y_2$}}] (g) [pleaf] {$v_3$} edge from parent node[above]{}
    	    }
    		child {node [fill=black!10!white!90, label=below:{\makebox[1cm][c]{$y_5$}}] (a2) [pleaf] {$v''_1$}
    		edge from parent node[above right]{}
    	    }
    	    edge from parent node[above left]{}
        }
    	child {node (j) [pnode] {\color{white}{$v$}}
    		child {node [label=below:{\makebox[1cm][c]{$y_2$}}] (k) [pleaf]  {\color{white}{$v$}} edge from parent node[above left]{}}
    		child {node [label=below:{\makebox[1cm][c]{$y_3$}}] (l) [pleaf]  {\color{white}{$v$}} edge from parent node[above right]{}
    		}
    		edge from parent node[above right]{}
    	};
\end{tikzpicture}}
&
\scalebox{0.75}{
\begin{tikzpicture}[scale = 1,every node/.style={scale=1},
		regnode/.style={circle,draw,minimum width=1.5ex,inner sep=0pt},
		leaf/.style={circle,fill=black,draw,minimum width=1.5ex,inner sep=0pt},
		pleaf/.style={rectangle,rounded corners=1ex,draw,font=\scriptsize,inner sep=3pt},
		pnode/.style={rectangle,rounded corners=1ex,draw,font=\scriptsize,inner sep=3pt},
		rootnode/.style={rectangle,rounded corners=1ex,draw,font=\scriptsize,inner sep=3pt},
		level/.style={sibling distance=6em/#1, level distance=4ex}
	]
	\node (z) [rootnode] {\color{white}{$v$}}
	child {node [fill=black!30!white!70] (a) [pnode] {$v_1$}
	    child {node [fill=black!10!white!90] (a1) [pnode] {$v'_1$}
	        child{node [fill=white!10!white!90, label=below:{\makebox[1cm][c]{$y_1$}}] (a11) [pleaf] {$v_2$}
	        edge from parent node[above left]{}
	        }
	        child{node [fill=white!10!white!90, label=below:{\makebox[1cm][c]{$y_2$}}] (a12) [pleaf] {$v_3$}
	        edge from parent node[above right]{}
	        }
    	edge from parent node[above left]{}
    	}
    	child {node [fill=black!10!white!90,
    	label=below:{\makebox[1cm][c]{$y_5$}}] (g) [pleaf] {$v''_1$}
    	edge from parent node[above]{}
        }
		edge from parent node[above left]{}
	}
	child {node (j) [pnode] {\color{white}{$v$}}
		child {node [label=below:{\makebox[1cm][c]{$y_2$}}] (k) [pleaf] {\color{white}{$v$}} edge from parent node[above left]{}}
		child {node [label=below:{\makebox[1cm][c]{$y_3$}}] (l) [pleaf] {\color{white}{$v$}}
			{
				child [grow=right] {node (s) {} edge from parent[draw=none]
					child [grow=up] {node (t) {} edge from parent[draw=none]
						child [grow=up] {node (u) {} edge from parent[draw=none]}
					}
				}
			}
			edge from parent node[above right]{}
		}
		edge from parent node[above right]{}
	};
\end{tikzpicture}}
&
\scalebox{0.8}{
\begin{tikzpicture}[scale = 1,every node/.style={scale=1},
    		regnode/.style={circle,draw,minimum width=1.5ex,inner sep=0pt},
    		leaf/.style={circle,fill=black,draw,minimum width=1.5ex,inner sep=0pt},
    		pleaf/.style={rectangle,rounded corners=1ex,draw,font=\scriptsize,inner sep=3pt},
    		pnode/.style={rectangle,rounded corners=1ex,draw,font=\scriptsize,inner sep=3pt},
    		rootnode/.style={rectangle,rounded corners=1ex,draw,font=\scriptsize,inner sep=3pt},
    		level/.style={sibling distance=6em/#1, level distance=4ex},
    		level 3/.style={sibling distance=3em, level distance=4ex}
    	]
	\node (z) [rootnode] {\color{white}{$v$}}
	child {node [fill=white!10!white!90] (a) [pnode] {$v_1$}
	    child {node [fill=black!30!white!70] (a1) [pnode] {$v_2$}
	        child{node [fill=black!10!white!90, label=below:{\makebox[1cm][c]{$y_1$}}] (a11) [pleaf] {$v'_2$}
	        edge from parent node[above left]{}
	        }
	        child{node [fill=black!10!white!90, label=below:{\makebox[1cm][c]{$y_5$}}] (a12) [pleaf] {$v''_2$}
	        edge from parent node[above right]{}
	        }
    	edge from parent node[above left]{}
    	}
    	child {node [fill=white!10!white!90,
    	label=below:{\makebox[1cm][c]{$y_2$}}] (g) [pleaf] {$v_3$}
    	edge from parent node[above]{}
        }
		edge from parent node[above left]{}
	}
	child {node (j) [pnode] {\color{white}{$v$}}
		child {node [label=below:{\makebox[1cm][c]{$y_2$}}] (k) [pleaf] {\color{white}{$v$}} edge from parent node[above left]{}}
		child {node [label=below:{\makebox[1cm][c]{$y_3$}}] (l) [pleaf] {\color{white}{$v$}}
			{
				child [grow=right] {node (s) {} edge from parent[draw=none]
					child [grow=up] {node (t) {} edge from parent[draw=none]
						child [grow=up] {node (u) {} edge from parent[draw=none]}
					}
				}
			}
			edge from parent node[above right]{}
		}
		edge from parent node[above right]{}
    	};
\end{tikzpicture}} \\

{\footnotesize (a) Tree $T_{t-1}$ after $t\!-\!1$ iterations.}
&
{\footnotesize (b) Variant 1: A leaf node $v_1''$ for label $j$ added as a child of an internal node $v_1$.}
&
{\footnotesize (c) Variant 2: A leaf node $v_1''$ for label $j$ and an internal node $v_1'$ (with all children of $v_1$ reassigned to it) added as children of $v_1$.}
&
{\footnotesize (d) Variant 3: A leaf node $v_2''$ for label $j$ and a leaf node $v_2'$ (with a reassigned label of $v_2$) added as children of $v_2$.} \\
\end{tabular}
\caption{%
Three variants of tree extension for a new label $j$. %
}
\label{fig:oplt-tree_building}
\vspace{-5pt}
\end{figure*}
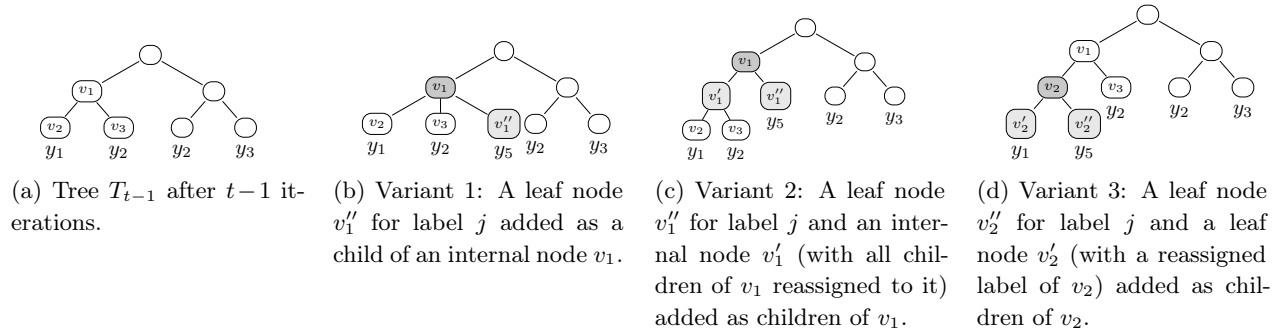

\begin{algorithm*}[t]
\caption{\footnotesize \Algo{OPLT.UpdateClassifiers}$(\bx, \calL_{\bx}, A_{\textrm{online}})$}
\label{alg:oplt-update_classifiers}
\begin{algorithmic}[1]
\footnotesize
    \State $(P, N) = \textsc{AssignToNodes}(T, \bx, \calL_{\bx})$ \CommentSS{Compute its positive and negative nodes}
    \For{$v \in P$} \CommentSS{For all positive nodes}
         \State $A_{\textrm{online}}\textsc{.Update}(\heta(v), (\bx, 1))$ \CommentSS{Update classifiers with a positive update with $\bx$.}
        \IfThen{$\hat{\theta}(v) \in \Theta$}{$A_{\textrm{online}}\textsc{.Update}(\hat \theta(v), (\bx, 1))$}
            \CommentSS{If aux. classifier exists, update it with a positive update with $\bx_i$.}
    \EndFor
    \ForDo{$v \in N$}{$A_{\textrm{online}}\textsc{.Update}(\heta(v), (\bx, 0))$} \CommentSS{Update all negative nodes with a negative update with $\bx$.}
\end{algorithmic}
\end{algorithm*}
\begin{algorithm*}[t]
\caption{\footnotesize \Algo{Random} and \Algo{Best-greedy} $A_{\textrm{policy}}(\bx, j, \calL_{\bx})$}
\label{alg:oplt-policy}
\begin{algorithmic}[1]
\footnotesize
\If{\Algo{RunFirstFor}$(\bx)$} \CommentSS{If the algorithm is run for the first time for the current observation $\bx$}
    \State $v = r_T$ \CommentSS{Set current node $v$ to root node}
    \While {$\childs{v} \nsubseteq L_T \land \childs{v} = b $} \CommentSS{While the node's children are not only leaf nodes and arity is equal to $b$}
        \If {\Algo{Random} policy}
            $v$ = $\textsc{SelectRandomly}(\childs{v})$ \CommentSS{If \Algo{Random} policy, randomly choose child node}
        \ElsIf{\Algo{Best-greedy} policy} \CommentSS{In the case of \Algo{Best-greedy} policy}
            \State $v = \argmax_{v' \in \childs{v}} (1 - \alpha)\heta(\bx, v') + \alpha {|L_{v'}|}^{-1}
            \left ( \log{|L_{v}|} - \log{ |\childs{v}|} \right )$ \CommentSS{Select child node with the best score}
        \EndIf
    \EndWhile
\Else~$v = $\Algo{GetSelectedNode()} \CommentSS{If the same $\bx$ is observed as the last time, select the node used previously}
\EndIf
\IfThen{$|\childs{v} \cap L_T| = 1$}{$v = v' \in \childs{v}: v' \in L_T$} \CommentSS{If node $v$ has only one leaf, change the selected node to this leaf}
\State \Algo{SaveSelectedNode}$(v)$ \CommentSS{Save the selected node $v$}
\State \Return{($v$, $|\childs{v}| = b_{\max} \lor v \subseteq L_T$)} \CommentSS{\scriptsize Return node $v$, if num. of $v$'s children reached the max. or $v$ is a leaf, insert a new node.}
\end{algorithmic}
\end{algorithm*}
\vspace{-5pt}

$A_{\textrm{policy}}$ returns the selected node $v$ and
a Boolean variable $insert$, which indicates whether an additional node $v'$ has to be added to the tree.
For the first variant, $v$ is an internal node, and $insert$ is set to false.
For the second variant, $v$ is an internal node, and $insert$ is set to true.
For the third variant, $v$ is a leaf node, and $insert$ is set to true.
In general, the policy can be as simple as selecting a random node or a node based on the current tree size to construct a complete tree.
It can also be much more complex, guided in general by $\bx$, current label $j$, and set $\calL_{\bx}$ of all labels of $\bx$.
Nevertheless, as mentioned before, the complexity of this step should be at most proportional to the complexity of updating the node classifiers for one label,
i.e., it should be proportional to the depth of the tree. We propose two such policies in the next subsection.

The \Algo{InsertNode} and \Algo{AddLeaf} procedures involve specific operations
initializing classifiers in the new nodes.
\Algo{InsertNode} is given in Algorithm~\ref{alg:oplt-insert_node}.
It inserts a new node $v'$ as a child of the selected node $v$.
If $v$ is a leaf, then its label is reassigned to the new node.
Otherwise, all children of $v$ become the children of $v'$.
In both cases, $v'$ becomes the only child of $v$.
Figure~\ref{fig:oplt-tree_building} illustrates inserting $v'$ as
either a child of an internal node (c) or a leaf node (d).
Since, the node classifier of $\node'$ aims at estimating $\eta(\bx, \node')$,
defined as $\prob(z_{\node'} = 1 \given z_{\pa{\node'}} = 1, \bx)$,
its both classifiers, $\heta(v')$ and $\hat\theta(v')$,
are initialized as copies (by calling the \textsc{Copy} function)
of the auxiliary classifier $\hat\theta(v)$ of the parent node $v$.
Recall that the task of auxiliary classifiers is to accumulate all positive updates in nodes,
so the conditioning $z_{\pa{\node'}} = 1$ is satisfied in that way.

Algorithm~\ref{alg:oplt-add_leaf} outlines the \Algo{AddLeaf} procedure.
It adds a new leaf node $v''$ for label $j$ as a child of node $v$.
The classifier $\heta(v'')$ is created as an ``inverse''
of the auxiliary classifier $\hat\theta(v)$ from node $v$.
More precisely, the \Algo{InverseClassifier} procedure creates a wrapper
inverting the behavior of the base classifier.
It predicts $1 - \heta$, where $\heta$ is the prediction of the base classifier,
and flips the updates, i.e., positive updates become negative and negative updates become positive.
Finally, the auxiliary classifier $\hat{\theta}(v'')$ of the new leaf node is initialized.
%

The final step in the main loop of \Algo{OPLT.Train} updates the node classifiers. The regular classifiers, $\heta(v) \in H_T$,
are updated exactly as in \Algo{IPLT.Train} given in Algorithm~\ref{alg:plt-incremental-learning}.
The auxiliary classifiers, $\theta(v) \in \Theta_T$,
are updated only in positive nodes according to their definition and purpose.

Notice that \Algo{OPLT.Train} can also be run without prior initialization with \Algo{OPLT.Init} if only a tree with properly trained node and auxiliary classifiers is provided. One can create such a tree using a set of already available observations $\calD$ and then learn node and auxiliary classifiers using the same \Algo{OPLT.Train} algorithm. Because all labels from $\calD$ should be present in the created tree, it is not updated by the algorithm. From now on, \Algo{OPLT.Train} can be used again to correctly update the tree for new observations.

\subsection{Random and best-greedy policy}
\label{sec:policies}

We discuss two policies $A_{\textrm{policy}}$ for \Algo{OPLT} that
can be treated as non-trivial generalization of the policy used in \Algo{CPET} to the \multilabel{} setting.
%
\Algo{CPET} builds a binary balanced tree by expanding leaf nodes, which corresponds to the use of the third variant of the tree structure extension only.
As the result, it gradually moves away labels that initially have been placed close to each other.
Particularly, labels of the first observed examples will finally end in leaves at the opposite sides of the tree.
This may result in lowering the predictive performance and increasing training and prediction times.
To address these issues, we introduce a solution, inspired by \citep{Prabhu_et_al_2018, Wydmuch_et_al_2018},
in which pre-leaf nodes, i.e., parents of leaf nodes, can be of much higher arity than the other internal nodes.
In general, we guarantee that arity of each pre-leaf node is upperbounded by $b_{\max}$, while all other internal nodes by $b$, where $b_{\max} \ge b$.

Both policies, presented jointly in Algorithm~\ref{alg:oplt-policy}, start with selecting one of the pre-leaves.
The first policy traverses a tree from top to bottom by randomly selecting child nodes.
The second policy, in turn, selects a child node using a trade-off between the balancedness of the tree and
fit of $\bx$, i.e., the value of $\heta(\bx, v)$:
$$
    \textrm{score}_v = (1 - \alpha)\heta(\bx, v) + \alpha \frac{1}{|L_v|} \log  \frac{|L_{\pa{v}}|}{ |\childs{\pa{v}}|}\,,
$$
where $\alpha$ is a trade-off parameter.
It is worth to notice that both policies work in logarithmic time of the number of internal nodes.
Moreover, we run this selection procedure only once for the current observation, regardless of the number of new labels.
If the selected node $v$ has fewer leaves than $b_{max}$, both policies follow the first variant of the tree extension,
i.e., they add a new child node with the new label assigned to node $v$.
Otherwise, the policies follow the second variant,
in which additionally, a new internal node is added as a child of $v$ with all its children inherited.
In case the selected node has only one leaf node among its children, which only happens after adding a new label with the second variant,
the policy changes the selected node $v$ to the previously added leaf node.

The above policies have two advantages over \Algo{CPET}.
Firstly, new labels coming with the same observation should stay close to each other in the tree.
Secondly, the policies allow for efficient management of auxiliary classifiers, which basically need to reside only in pre-leaf nodes,
with the exception of leaf nodes added in the second variant.
The original \Algo{CPET} algorithm needs to maintain auxiliary classifiers in all leaf nodes.

\subsection{Theoretical analysis of OPLT}
\label{sec:oplt-theory}

The \Algo{OPLT} algorithm has been designed to satisfy the properness and efficiency property.
The theorem below states this fact formally.
\begin{restatable}{theorem}{thmoplt}
\label{thm:oplt}
\Algo{OPLT} is a proper and efficient online \Algo{PLT} algorithm.
\end{restatable}
We present the proof in Appendix~\ref{app:oplt}.
To show the properness, it uses induction for both the outer and inner loop of the algorithm,
where the outer loop iterates over observations $(\bx_t, \calL_{\bx_t})$,
while the inner loop over new labels in $\calL_{\bx_{t}}$.
The key elements to prove this property are the use of the auxiliary classifiers
and the analysis of the three variants of the tree structure extension.
The efficiency is proved by noticing that the algorithm creates up to two new nodes per new label,
each node having at most two classifiers.
Therefore, the number of updates is no more than twice the number of updates in \Algo{IPLT}.
Moreover, any node selection policy in which cost is proportional to the cost
of updating \Algo{IPLT} classifiers for a single label
meets the efficiency requirement.
Notably, the policies presented above satisfy this constraint.
Note that training of \Algo{IPLT} can be performed in logarithmic time in the number of labels
under the additional assumption of using a balanced tree with constant nodes arity~\citep{Busa-Fekete_et_al_2019}.
Because presented policies aim to build trees close to balanced,
the time complexity of the \Algo{OPLT} training should also be close to logarithmic in the number of labels.


\section{Experiments}
\label{sec:experiments}


%
\begin{table}[ht]
    \caption{Datasets used for experiments on extreme multi-label classification task and few-shot multi-class classification task.
    Notation: $N$ -- number of samples, $m$ -- number of labels, $d$ -- number of features, $S$ -- shot}
    \label{tab:datasets}
    \begin{center}
    \footnotesize
    \begin{tabular}{
        l c r r r r
    }
    \toprule
    Dataset & $N_{train}$ & $N_{test}$ & $m$ & $d$ \\

    \midrule

    AmazonCat  & 1186239 & 306782 & 13330 & 203882 \\
    Wiki10 & 14146 & 6616 & 30938 & 101938 \\
    WikiLSHTC & 1778351 & 587084 & 325056 & 1617899 \\
    Amazon & 490449 & 153025 & 670091 & 135909 \\
    \midrule
    ALOI & 97200 & 10800 & 1001 & 129 \\
    WikiPara-$S$ & $S\times 10000$ & 10000 & 10000 & 188084 \\

    \bottomrule
    \end{tabular}
    \end{center}
    \vspace{-10pt}
\end{table}

In this section, we empirically compare \Algo{OPLT} and \Algo{CMT} on two tasks, extreme multi-label classification and few-shot multi-class classification.
We implemented \Algo{OPLT} in C++.%
\footnote{Repository with the code and scripts to reproduce the experiments: \url{https://github.com/mwydmuch/napkinXC}}
We use online logistic regression to train node classifiers with the \Algo{AdaGrad}~\citep{Duchi_et_al_2011} updates.

For \Algo{CMT}, we use a Vowpal Wabbit~\citep{Langford_et_al_2007} implementation (also in C++), provided by courtesy of its authors.
It uses linear models also incrementally updated by \Algo{AdaGrad}, but all model weights are stored in one large continuous array using the hashing trick.
However, it requires at least some prior knowledge about the size of the feature space since the size of the array must be determined beforehand, which can be hard in a fully online setting.
%
To address the problem of unknown features space, we store weights of \Algo{OPLT} in an easily extendable hash map based on Robin Hood Hashing~\citep{Celis_et_al_1985}, which ensures very efficient insert and find operations.
Since the model sparsity increases with the depth of a tree for sparse data,
this solution might be much more efficient in terms of used memory than the hashing trick and does not negatively impact predictive performance.

For all experiments, we use the same, fixed hyper-parameters for \Algo{OPLT}. We set learning rate to $1$,
\Algo{Adagrad}'s $\epsilon$ to $0.01$ and the tree balancing parameter $\alpha$ to $0.75$,
since more balanced trees yield better predictive performance
(see Appendix~\ref{app:tree-alpha} for empirical evaluation of the impact of parameter $\alpha$ on precision at $k$,
train and test times, and the tree depth).
The only exception is the degree of pre-leaf nodes, which we set to $100$ in the XMLC experiment, and to $10$
in the few-shot multi-class classification experiment.
For \Algo{CMT} we use hyper-parameters suggested by the authors.
According to the appendix of~\citep{Sun_et_al_2019},
\Algo{CMT} achieves the best predictive performance after 3 passes over training data.
For this reason, we give all algorithms the maximum of 3 such passes and report the best results
(see Appendix~\ref{app:detail-results} and \ref{app:detail-results-fs} for the detailed results after 1 and 3 passes).
We repeated all the experiments 5 times,
each time shuffling the training set and report the mean performance.
We performed all the experiments on an Intel Xeon E5-2697 v3 2.6GHz machine with 128GB of memory.

\subsection{Extreme multi-label classification}
\label{sec:xmlc-experiments}

\begin{table*}[ht!]
    \tabcolsep=4pt
    \caption{Mean precision at $\{1,3,5\}$ (\%) and CPU train time of \Algo{Parabel}, \Algo{IPLT}, \Algo{CMT}, \Algo{OPLT} for XMLC tasks.
    }
    \label{tab:xmlc-precision-results}
    \begin{center}
    \resizebox{\textwidth}{!}{
    \begin{tabular}{
        l
        | r r r r
        | r r r r
        | r r r r
        | r r r r
    }
    \toprule
    & \multicolumn{4}{c|}{AmazonCat}
    & \multicolumn{4}{c|}{Wiki10}
    & \multicolumn{4}{c|}{WikiLSHTC}
    & \multicolumn{4}{c}{Amazon}
    \\

    Algo.
    & \multicolumn{1}{c}{P$@1$} & \multicolumn{1}{c}{P$@3$} & \multicolumn{1}{c}{P$@5$} & \multicolumn{1}{c|}{$t_{train}$}
    & \multicolumn{1}{c}{P$@1$} & \multicolumn{1}{c}{P$@3$} & \multicolumn{1}{c}{P$@5$} & \multicolumn{1}{c|}{$t_{train}$}
    & \multicolumn{1}{c}{P$@1$} & \multicolumn{1}{c}{P$@3$} & \multicolumn{1}{c}{P$@5$} & \multicolumn{1}{c|}{$t_{train}$}
    & \multicolumn{1}{c}{P$@1$} & \multicolumn{1}{c}{P$@3$} & \multicolumn{1}{c}{P$@5$} & \multicolumn{1}{c}{$t_{train}$}
    \\
    \midrule

    \Algo{Parabel} & 92.58 & 78.53 & 63.90 & 10.8m & 84.17 & 72.12 & 63.30 & 4.2m & 62.78 & 41.22 & 30.27 & 14.4m & 43.13 & 37.94 & 34.00 & 7.2m \\

    \Algo{IPLT} & 93.11 & 78.72 & 63.98 & 34.2m & 84.87 & 74.42 & 65.31 & 18.3m & 60.80 & 39.58 & 29.24 & 175.1m & 43.55 & 38.69 & 35.20 & 79.4m \\

    \midrule

    \Algo{CMT}
    & 89.43 & 70.49 & 54.23 & 168.2m
    & 80.59 & 64.17 & 55.25 & 35.1m
    & - & - & - & -
    & - & - & - & - \\

    \Algo{OPLT}$_R$ & 92.66 & 77.44 & 62.52 & 99.5m & 84.34 & 73.73 & 64.31 & 30.3m & 47.76 & 30.97 & 23.37 & 330.1m & 38.42 & 34.33 & 31.32 & 134.2m \\

    \Algo{OPLT}$_B$ & 92.74 & 77.74 & 62.91 & 84.1m & 84.47 & 73.73 & 64.39 & 27.7m & 54.69 & 35.32 & 26.31 & 300.0m & 41.09 & 36.65 & 33.42 & 111.9m \\

    \Algo{OPLT-W} & 93.14 & 78.68 & 63.92 & 43.7m & 85.22 & 74.68 & 64.93 & 28.2m & 59.23 & 38.39 & 28.38 & 205.7m & 42.21 & 37.60 & 34.25 & 98.3m \\

    \bottomrule
    \end{tabular}}
    \end{center}
    \vspace{-5pt}
\end{table*}

\begin{figure*}[ht!]

    \centering

\scriptsize
\begin{tabular}{ccc}
\tabcolsep=0pt
\begin{tikzpicture}[]
            \begin{axis}[
                title=ALOI,
                xlabel=Number of examples,
            	ylabel=Entropy reduction (bits),
            	y label style={at={(axis description cs:0.15,0.5)},anchor=south},
            	x label style={at={(axis description cs:0.5,-0.15)},anchor=south},
            	ymin=3,
                ymax=10,
                width=6.25cm,
                height=3.5cm,
                grid=both,
            	legend pos=south east,
            	legend style={nodes={scale=0.9, transform shape}}]
            	\addplot[color=red] coordinates {
(10000,8.221587121264804)
(15000,8.344886182346173)
(20000,8.4641372862766)
(25000,8.524502537449163)
(30000,8.568081521704045)
(35000,8.615236037406826)
(40000,8.655709595518514)
(45000,8.69673459039983)
(50000,8.741938818829311)
(55000,8.782552722381924)
(60000,8.813245691686358)
(65000,8.846599871882102)
(70000,8.87743752197062)
(75000,8.905769065519532)
(80000,8.930258330599084)
(85000,8.955394180608101)
(90000,8.979376604487056)
(95000,9.00189792382318)
(97199,9.01231445878839)

            	};
            	\addplot[color=blue,style=densely dashed] coordinates {
(10000,7.277054876148728)
(15000,7.753998902448721)
(20000,8.017086773362719)
(25000,8.186659017202114)
(30000,8.315905816354034)
(35000,8.420743443367607)
(40000,8.497752062059769)
(45000,8.564614431690005)
(50000,8.621319301014191)
(55000,8.665076493029387)
(60000,8.70315262306323)
(65000,8.7354917146641)
(70000,8.764111903020595)
(75000,8.790859988500507)
(80000,8.811576993822163)
(85000,8.832181147742155)
(90000,8.851714691125197)
(95000,8.870007084180726)
(97199,8.876501598663777)

            	};
            	\addplot[color=violet,style=densely dashdotted] coordinates {
(10000,7.451211111832329)
(15000,7.904887466246706)
(20000,8.164906926675688)
(25000,8.341630009329911)
(30000,8.4870360800319)
(35000,8.585822601483166)
(40000,8.665424750351725)
(45000,8.728070160431894)
(50000,8.78352213467528)
(55000,8.827298239640111)
(60000,8.86578584398965)
(65000,8.897334221909615)
(70000,8.926209993952314)
(75000,8.952837326603353)
(80000,8.975561406273616)
(85000,8.99664860701806)
(90000,9.015415052386688)
(95000,9.033133882914305)
(97199,9.039914270047351)

            	};
            	\legend{\Algo{CMT}, \Algo{OPLT$_R$}, \Algo{OPLT$_B$}}
            \end{axis}
        \end{tikzpicture}
&%
\begin{tikzpicture}[]
            \begin{axis}[
                title=Wikipara-3,
                xlabel=Number of examples,
            	x label style={at={(axis description cs:0.5,-0.15)},anchor=south},
            	grid=both,
                xtick distance=5000,
                ymin=3,
                ymax=10,
                width=6.25cm,
                height=3.5cm]
            	\addplot[color=red] coordinates {
(10000,6.9541963103868705)
(15000,7.1497471195046804)
(20000,7.379378367071266)
(25000,7.529820946528699)
(29999,7.660566437710389)

            	};
            	\addplot[color=blue,style=densely dashed] coordinates {
(10000,6.658211482751794)
(15000,7.291557515404494)
(20000,7.636624620543648)
(25000,7.944858445807538)
(29999,8.294666693861961)

            	};
            	\addplot[color=violet,style=densely dashdotted] coordinates {
(10000,4.087462841250339)
(15000,5.437404202541455)
(20000,6.3750394313469245)
(25000,7.1001366712854495)
(29999,7.65826147626818)

            	};
            \end{axis}
        \end{tikzpicture}
&%
\begin{tikzpicture}[]
            \begin{axis}[
                title=Wikipara-5,
                xlabel=Number of examples,
            	x label style={at={(axis description cs:0.5,-0.15)},anchor=south},
            	grid=both,
                xtick distance=5000,
                ymin=3,
                ymax=10,
                width=6.25cm,
                height=3.5cm]
            	\addplot[color=red] coordinates {
(10000,7.199672344836368)
(15000,7.327956766630705)
(20000,7.535275376620799)
(25000,7.663913842115975)
(30000,7.798763388822395)
(35000,7.913727173005174)
(40000,7.981567281903017)
(45000,8.03342300153745)
(49999,8.111657346263133)

            	};
            	\addplot[color=blue,style=densely dashed] coordinates {
(10000,7.179909090014934)
(15000,7.781359713524659)
(20000,8.247927513443585)
(25000,8.632268215499511)
(30000,8.933690654952233)
(35000,9.222623459463485)
(40000,9.485829308701904)
(45000,9.687667867397796)
(49999,9.895755513337479)

            	};
            	\addplot[color=violet,style=densely dashdotted] coordinates {
(10000,4.584962500721156)
(15000,6.437405867189899)
(20000,7.3750394313469245)
(25000,7.977279923499916)
(30000,8.353146825498083)
(35000,8.78603760140825)
(40000,9.088125692188207)
(45000,9.370687406807217)
(49999,9.620982217595696)
            	};
            \end{axis}
        \end{tikzpicture}
\end{tabular}
    \caption{Online progressive performance of CMT and OPLT with respect to the number of samples on few-shot multi-class classification tasks.}
    \label{fig:online-few-shot}
    \vspace{-5pt}
\end{figure*}
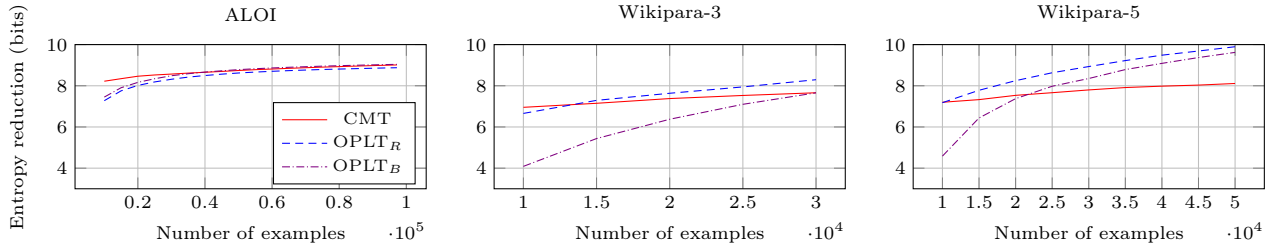

\begin{table}[ht!]
    \tabcolsep=4pt
    \caption{Mean accuracy of prediction (\%) and train CPU time of \Algo{CMT}, \Algo{OPLT} for few-shot multi-class classification tasks.}
    \label{tab:fs-results}
    \begin{center}
    \footnotesize
    \begin{tabular}{
        l
        | r r
        | r r
        | r r
    }
    \toprule
    & \multicolumn{2}{c|}{ALOI}
    & \multicolumn{2}{c|}{Wikipara-3}
    & \multicolumn{2}{c}{Wikipara-5}
    \\

    Algo.
    & \multicolumn{1}{c}{Acc.} & \multicolumn{1}{c|}{$t_{train}$}
    & \multicolumn{1}{c}{Acc.} & \multicolumn{1}{c|}{$t_{train}$}
    & \multicolumn{1}{c}{Acc.} & \multicolumn{1}{c}{$t_{train}$}
    \\
    \midrule
    \Algo{CMT} & 71.98 & 207.1s & 3.28 & 63.1s & 4.21 & 240.9s \\
    \Algo{OPLT}$_R$ & 66.50 & 20.3s & 27.34 & 16.4s & 40.67 & 27.6s \\
    \Algo{OPLT}$_B$ & 67.26 & 18.1s & 24.66 & 15.6s & 39.13 & 27.2s \\

    \bottomrule
    \end{tabular}
    \end{center}
    \vspace{-10pt}
\end{table}

In the XMLC setting, we compare performance in terms of precision at $\{1,3,5\}$ and the training time
(see Appendix~\ref{app:detail-results} for prediction times and propensity scored precision at $\{1,3,5\}$)
on four benchmark datasets: AmazonCat, Wiki10, WikiLSHTC and Amazon, taken from the XMLC repository~\citep{Bhatia_et_al_2016}.
We use the original train and test splits.
Statistics of these datasets are included in Table~\ref{tab:datasets}.
In this setting, \Algo{CMT} has been originally used to augment an online one-versus-rest (OVR) algorithm.
In other words, it can be treated as a specific index that enables fast prediction and speeds up training by performing a kind of negative sampling.
In addition to \Algo{OPLT} and \Algo{CMT} we also report results of \Algo{IPLT} and \Algo{Parabel}~\citep{Prabhu_et_al_2018}.
\Algo{IPLT} is implemented similarly to \Algo{OPLT}, but uses a tree structure built-in offline mode.
\Algo{Parabel} is, in turn, a fully batch variant of $\Algo{PLT}$. 
Not only the tree structure, but also node classifiers are trained in the batch mode using the \Algo{LIBLINEAR} library~\citep{liblinear}.
We use a single tree variant of this algorithm.
Both \Algo{IPLT} and \Algo{Parabel} are used with the same tree building algorithm, which is based on a specific hierarchical 2-means clustering of labels~\citep{Prabhu_et_al_2018}.
Additionally, we report the results of an \Algo{OPLT} with warm-start (\Algo{OPLT-W}) that is first trained on a sample of 10\% of training examples and a tree created using hierarchical 2-means clustering on the same sample. After this initial phase, \Algo{OPLT-W} is trained on the remaining 90\% of data using the \Algo{Best-greedy} policy
(see Appendix~\ref{app:warm-start} for the results of \Algo{OPLT-W} trained with different sizes of the warm-up sample
and comparison with \Algo{IPLT} trained only on the same warm-up sample).

Results of the comparison are presented in Table~\ref{tab:xmlc-precision-results}.
Unfortunately, \Algo{CMT} does not scale very well in the number of labels nor in the number of examples,
resulting in much higher memory usage for massive datasets.
Therefore, we managed to obtain results only for Wiki10 and AmazonCat datasets using all available 128GB of memory.
\Algo{OPLT} with both extension policies achieves results as good as \Algo{Parabel} and \Algo{IPLT} and significantly outperforms \Algo{CMT} on AmazonCat and Wiki10 datasets.
For larger datasets \Algo{OPLT} with \Algo{Best-greedy} policy outperforms the \Algo{Random} policy but obtains worse results than its offline counterparts, with trees built with hierarchical 2-means clustering, especially on the WikiLSHTC dataset. \Algo{OPLT-W}, however, achieves results almost as good as \Algo{IPLT} what proves that good initial structure, even with only some labels, helps to build a good tree in an online way. In terms of training times, \Algo{OPLT}, as expected, is slower than \Algo{IPLT} due to additional auxiliary classifiers and worse tree structure, both leading to a larger number of updates.

\subsection{Few-shot multi-class classification}
\label{sec:fs-experiments}

In the second experiment we compare \Algo{OPLT} with \Algo{CMT} on three few-shot learning multi-class datasets: ALOI~\citep{Geusebroek_et_al_2005}, 3 and 5-shot versions of WikiPara. Statistics of these datasets are also included in Table~\ref{tab:datasets}.
\Algo{CMT} has been proven in~\citep{Sun_et_al_2019} to perform better than two other logarithmic-time online multi-class algorithms, \Algo{LOMTree}~\citep{Choromanska_Langford_2015} and \Algo{Recall Tree}~\citep{Daume_et_al_2017} on these specific datasets. We use here the same version of \Algo{CMT} as used in a similar experiment in the original paper~\citep{Sun_et_al_2019}.

Since \Algo{OPLT} and \Algo{CMT} operate online, we compare their performance in two ways: 1) using online progressive validation~\citep{Blum_et_al_1999}, where each example is tested ahead of training and 2) using offline evaluation on the test set after seeing whole training set.
Figure~\ref{fig:online-few-shot} summarizes the results in terms of progressive performance.
In the same fashion as in \citep{Sun_et_al_2019}, we report entropy reduction of accuracy from the constant predictor, calculated as $\log_2 (\mathrm{Acc_{algo}}) - \log_2(\mathrm{Acc_{const}})$, where $\mathrm{Acc_{algo}}$ and $\mathrm{Acc_{const}}$ is mean accuracy of the evaluated algorithm and the constant predictor.
In Table~\ref{tab:fs-results} we report results on the test datasets.
In online and offline evaluation \Algo{OPLT} performs similar to \Algo{CMT} on ALOI dataset,
while it significantly dominates on the WikiPara datasets. 


\section{Conclusions}
\label{sec:conclusions}

In this paper, we introduced online probabilistic label trees, an algorithm that trains a label tree classifier in a fully online manner, without any prior knowledge about the number of training instances, their features and labels. \Algo{OPLT}s can be used for both multi-label and multi-class classification. They outperform \Algo{CMT} in almost all experiments, scaling at the same time much more efficiently on tasks with a large number of examples, features and labels.

\subsubsection*{Acknowledgements}

Computational experiments have been performed in Poznan Supercomputing and Networking Center.

\bibliography{references}

\newpage
\pagebreak
\appendix
\onecolumn

\section{Training in \Algo{PLT}}
\label{app:training}

The pseudocode with a brief description of the training algorithm for \Algo{IPLT}s is given in the main text.
Here we discuss it in more detail.
The algorithm first initializes all node classifiers.
Training relies on a proper assignment of training examples to nodes.
To train probabilistic classifiers $\heta(v)$, $v \in V_T$,
in all nodes of a tree $T$, we need to properly filter training examples as given in (\ref{eqn:plt_factorization}).
In each iteration the algorithm identifies for a given training example a set of \emph{positive} and \emph{negative nodes},
i.e., the nodes for which a training example is treated respectively
as positive (i.e, $(\bx, z_v = 1)$), or negative (i.e., $(\bx, z_v = 0)$).
The \textsc{AssignToNodes} method is given in Algorithm~\ref{alg:plt-assign}.
It initializes the positive nodes to the empty set and the negative nodes to the root node (to deal with $\by$ of all zeros).
Next, it traverses the tree from the leaves corresponding to the labels of the training example to the root
adding the visited nodes to the set of positive nodes.
It also removes each visited node from the set of negative nodes, if it has been added to this set before.
All children of the visited node, which are not in the set of positive nodes, are then added to the set of negative nodes.
If the parent node of the visited node has already been added to positive nodes, the traversal on this path stops.

Using the positive and negative nodes the procedure updates the corresponding classifiers
with an online learner $A_{\textrm{online}}$ of choice.
Note that all node classifiers are trained in fact independently as updates in any node do not influence training of the other nodes.
The output of the algorithm is a set of probabilistic classifiers $H$.

%
\begin{algorithm}[H]
\caption{\Algo{IPLT/OPLT.AssignToNodes}$(T, \bx, \calL_{\bx})$}
\label{alg:plt-assign}
\begin{small}
\begin{algorithmic}[1]
\State $P = \emptyset$, $N = \{r_T\}$ \Comment{Initialize sets of positive and negative nodes}
\For{$j \in \calL_{\bx}$} \Comment{For all labels of the training example}
\State $v = \ell_j$  \Comment{Set $v$ to a leaf corresponding to label $j$}
\While{$v$ not null \textbf{and} $v \not \in P$} \Comment{On a path to the root or the first positive node (excluded)}
\State $P = P \cup \{v\}$ \Comment{Assign a node to positive nodes}
\State $N = N \setminus \{v\}$ \Comment{Remove the node from negative nodes if added there before}
\For{$v' \in \childs{v}$} \Comment{For all its children}
\If{$v' \not \in P$} \Comment{If a child is not a positive node}
\State $N = N \cup \{v'\}$ \Comment{Assign it to negative nodes}
\EndIf
\EndFor
\State $v = \pa{v}$ \Comment{Move up along the path}
\EndWhile
\EndFor
\State \textbf{return} $(P,N)$ \Comment{Return a set of positive and negative nodes for the training example}
\end{algorithmic}
\end{small}
\end{algorithm}
\vspace{-10pt}

\section{Prediction in \Algo{PLT}}
\label{app:prediction}

Algorithm~\ref{alg:ucs_prediction} outlines the prediction procedure for \Algo{PLT}s that returns top $k$ labels.
It is based on the uniform-cost search. Alternatively, one can use beam search.
\begin{algorithm}[h!]
\caption{\Algo{IPLT/OPLT.PredictTopLabels}$(T, H, k, \bx)$}
\label{alg:ucs_prediction}
\begin{small}
\begin{algorithmic}[1]
\State $\hat\by = \vec{0}$, $\calQ = \emptyset$, \Comment{Initialize prediction vector to all zeros and a priority queue, ordered descending by $\heta_v(\bx)$}
\State $k' = 0$ \Comment{Initialize counter of predicted labels}
\State $\calQ\mathrm{.add}((r_T, \heta(\bx, r_T)))$ \Comment{Add the tree root with the corresponding estimate of probability}
\While{$k' < k$}   \Comment{While the number of predicted labels is less than $k$}
	\State $(v, \heta_v(\bx)) = \calQ\mathrm{.pop}()$ \Comment{Pop the top element from the queue}
	\If{$v$ is a leaf}  \Comment{If the node is a leaf}
		\State $\hy_v = 1$ \Comment{Set the corresponding label in the prediction vector}
		\State $k' = k' + 1$ \Comment{Increment the counter}
	\Else \Comment{If the node is an internal node}
	\For{$v' \in \childs{v}$} \Comment{For all child nodes}
	    \State $\heta_{v'}(\bx) = \heta_v(\bx) \times \heta(\bx,v')$ \Comment{Compute $\heta_{v'}(\bx)$ using $\heta(v') \in H$}
		\State $\calQ\mathrm{.add}((v', \heta_{v'}(\bx)))$  \Comment{Add the node and the computed probability estimate}
	\EndFor
	\EndIf
\EndWhile
\State \textbf{return} $\hat\by$ \Comment{Return the prediction vector}
\end{algorithmic}
\end{small}
\end{algorithm}
\vspace{-10pt}

\section{The proof of the result from Section~\ref{sec:oplt-theory}}
\label{app:oplt}

Theorem~\ref{thm:oplt} concerns two properties, properness and efficency, of an \Algo{OPLT} algorithm.
We first prove that the \Algo{OPLT} algorithm satisfies each of the property in two separate lemmas.
The final proof of the theorem is then straight-forward.

\begin{lemma}
\label{lem:proper}
\Algo{OPLT} is a proper \Algo{OPLT} algorithm.
\end{lemma}
\begin{proof}
We need to show that for any $\calS$ and $t$ the two of the following hold.
Firstly, that the set $L_{T_t}$ of leaves of tree $T_t$ built by \Algo{OPLT} correspond to $\calL_t$,
the set of all labels observed in $S_t$.
Secondly, that the set $H_t$ of classifiers trained by \Algo{OPLT}
is exactly the same as $H = \textsc{IPLT.Train}(T_t, A_{\textrm{online}}, \mathcal{S}_t)$,
i.e., the set of node classifiers trained incrementally by Algorithm~\ref{alg:plt-incremental-learning}
on $\calD = \calS_t$ and tree $T_t$ given as input parameter.
We will prove it by induction with the base case for $\calS_0$ and
the induction step for $\calS_t$, $t \ge 1$, with the assumption that the statement holds for $\calS_{t-1}$.

For the base case of $\calS_0$, tree $T_0$ is initialized with the root node $r_T$ with no label assigned
and set $H_0$ of node classifiers with a single classifier assigned to the root.
As there are no observations, this classifier receives no updates.
Now, notice that $\textsc{IPLT.Train}$, run on $T_0$ and $\calS_0$,
returns exactly the same set of classifiers $H$ that contains solely the initialized root node classifier
without any updates  (assuming that initialization procedure is always the same).
There are no labels in any sequence of 0 observations and also $T_0$ has no label assigned.

The induction step is more involved as we need to take into account
the internal loop which extends the tree with new labels.
Let us consider two cases.
In the first one, observation $(\bx_t, \calL_{\bx_t})$ does not contain any new label.
This means that that the tree $T_{t-1}$ will not change, i.e., $T_{t-1} = T_t$.
Moreover, node classifiers from $H_{t-1}$ will get the same updates for $(\bx_t, \calL_{\bx_t})$
as classifiers in \Algo{IPLT.Train}, therefore $H_t =  \Algo{IPLT.Train}(T_t, A_{\textrm{online}}, \mathcal{S}_t)$.
It also holds that $l_j \in L_{T_t}$ iff $j \in \labels_t$, since $\labels_{t-1} = \labels_t$.
In the second case, observation $(\bx_t, \calL_{\bx_t})$ has $m' = | \calL_{\bx_t} \setminus \labels_{t-1}|$ new labels.
Let us make the following assumption for the \Algo{UpdateTree} procedure,
which we later prove that it indeed holds.
Namely, we assume that the set $H_{t'}$ of classifiers after calling the \Algo{UpdateTree} procedure
is the same as the one being returned by $\Algo{IPLT.Train}(T_t, A_{\textrm{online}}, \mathcal{S}_{t-1})$,
where $T_t$ is the extended tree.
Moreover, leaves of $T_t$ correspond to all observed labels seen so far.
If this is the case, the rest of the induction step is the same as in the first case.
All updates to classifiers in $H_{t'}$ for $(\bx_t, \calL_{\bx_t})$ are the same as in \Algo{IPLT.Train}.
Therefore $H_t =  \Algo{IPLT.Train}(T_t, A_{\textrm{online}}, \mathcal{S}_t)$.

Now, we need to show that the assumption for the \Algo{UpdateTree} procedure holds.
To this end we also use induction, this time on the number $m'$ of new labels.
For the base case, we take $m' = 1$.
The induction step is proved for $m' > 1$ with the assumption that the statement holds for $m' - 1$.

For $m' = 1$ we need consider two scenarios.
In the first scenario, the new label is the first label in the sequence.
This label will be then assigned to the root node $r_T$.
So, the structure of the tree does not change, i.e., $T_{t-1} = T_t$.
Furthermore, the set of classifiers also does not changed,
since the root classifier has already been initialized.
It might be negatively updated by previous observations.
Therefore, we have $H_{t'} = \Algo{IPLT.Train}(T_t, A_{\textrm{online}}, \mathcal{S}_{t-1})$.
Furthermore, all observed labels are appropriately assigned to the leaves of $T_t$.
In the second scenario, set $\labels_{t-1}$ is not empty.
We need to consider in this scenario the three variants of tree extension illustrated in Figure~\ref{fig:oplt-tree_building}.

In the first variant, tree $T_{t-1}$ is extended by one leaf node only without any additional ones.
\Algo{AddNode} creates a new leaf node $v''$ with the new label assigned to the tree.
After this operation the tree contains all labels from $\calS_t$.
The new leaf $v''$ is added as a child of the selected node $v$.
This new node is initialized as $\heta(v'') = \textsc{InverseClassifier}(\hat\theta(v))$.
Recall that \Algo{InverseClassifier} creates a wrapper that inverts the behavior of the base classifier.
It predicts $1 - \heta$, where $\heta$ is the prediction of the base classifier,
and flips the updates, i.e., positive updates become negative and negative updates become positive.
From the definition of the auxiliary classifier,
we know that $\hat\theta(v)$ has been trained on all positives updates of $\heta(v)$.
So, $\heta(v'')$  is initialized with a state as
if it was updated negatively each time $\heta(v)$ was updated positively in sequence $S_{t-1}$.
Notice that in $S_{t-1}$ there is no observation labeled with the new label.
Therefore $\heta(v'')$ is the same as if it was created and updated using $\textsc{IPLT.Train}$.
There are no other operations on $T_{t-1}$,
so we have that $H_{t'} = \Algo{IPLT.Train}(T_t, A_{\textrm{online}}, \mathcal{S}_{t-1})$.

In the second variant, tree $T_{t-1}$ is extended by internal node $v'$ and leaf node $v''$.
The internal node $v'$ is added in \Algo{InsertNode}.
It becomes a parent of all child nodes of the selected node $v$ and the only child of this node.
Thus, all leaves of the subtree of $v$ does not change.
Since $v'$ is the root of this subtree,
its classifier $\heta(v')$ should be initialized as a copy of the auxiliary classifier $\hat\theta(v)$,
which has accumulated all updates from and only from observations with labels assigned to the leaves of this subtree.
Addition of the leaf node $v''$ can be analyzed as in the first variant.
Since nothing else has changed in the tree and in the node classifiers,
we have that $H_{t'} = \Algo{IPLT.Train}(T_t, A_{\textrm{online}}, \mathcal{S}_{t-1})$.
Moreover, the tree contains the new label, so the statement holds.

The third variant is similar to the second one.
Tree $T_{t-1}$ is extended by two leaf nodes $v'$ and $v''$
being children of the selected node $v$.
Insertion of leaf $v'$ is similar to insertion of node $v'$ in the second variant,
with the difference that $v$ does not have any children and
its label has to be reassigned to $v'$.
The new classifier in $v'$ is initialized as a copy of the auxiliary classifier $\hat\theta(v)$,
which contains all updates from and only from observations with the label assigned previously to $v$.
Insertion of $v''$ is exactly the same as in the second variant.
From the above, we conclude that $H_{t'} = \Algo{IPLT.Train}(T_t, A_{\textrm{online}}, \mathcal{S}_{t-1})$
and that $T_t$ contains all labels from $T_{t-1}$ and the new label.
In this way we prove the base case.

The induction step is similar to the second scenario of the base case.
The only difference is that we do not extent tree $T_{t-1}$,
but an intermediate tree with $m'-1$ new labels already added.
Because of the induction hypothesis,
the rest of the analysis of the three variants of tree extension is exactly the same.
This ends the proof that the assumption for the inner loop holds.
At the same time it finalizes the entire proof.
\end{proof}

\begin{lemma}
\label{lem:efficient}
\Algo{OPLT} is an efficient \Algo{OPLT} algorithm.
\end{lemma}
\begin{proof}
The \Algo{OPLT} maintains one additional classifier per each node in comparison to \Algo{IPLT}.
Hence, for a single observation there is at most one update more for each positive node.
Furthermore, the time and space cost of the complete tree building policy is constant per a single label,
if implemented with an array list.
In this case, insertion of any new node can be made in amortized constant time,
and the space required by the array list is linear in the number of nodes.
Concluding the above, the time and space complexity of \Algo{OPLT} is in constant factor of $C_t$ and $C_s$,
the time and space complexity of \Algo{IPLT} respectively.
This proves that \Algo{OPLT} is an efficient \Algo{OPLT} algorithm.
\end{proof}

\thmoplt*
\begin{proof}
The theorem directly follows from Lemma~\ref{lem:proper} and Lemma~\ref{lem:efficient}.
\end{proof}

\section{Detailed results of \Algo{OPLT} and \Algo{CMT} on extreme multi-label classification tasks}
\label{app:detail-results}


In Table~\ref{tab:xmlc-psp-results}, complementary to Table~\ref{tab:xmlc-precision-results}, we report performance on propensity scored precision at $\{1,3,5\}$~\citep{Jain_et_al_2016} defined as:
$$
\mathrm{PSP}@k = \frac{1}{k} \sum_{j \in \hat \calL_{\bx}} q_j \assert{\tilde{y}_j = 1} \,,
$$
where $q_j = 1 + C(N_j + B)^{-A}$ is inverse propensity of label $j$, where $N_j$ is the number of data points annotated with label $j$ in the observed ground truth dataset of size $N$, $A$, $B$ are dataset specific parameters and $C = (\log N - 1)(B + 1)^A$. For Wiki10 and AmazonCat datasets we use $A=0.55, B=1.5$, for WikiLSHTC $A=0.5, B=0.4$, and for Amazon $A=0.6, B=2.6$ as recommended in \citep{Jain_et_al_2016}.
Since values of $q_j$ are higher for infrequent labels, propensity scored precision at $k$ promotes the accurate prediction on harder to predict tail labels. As in the case of the results in terms of standard precision at $k$, \Algo{OPLT} outperforms \Algo{CMT}, being slightly worse than its offline counterparts \Algo{IPLT} and \Algo{Parabel}, especially on the WikiLSHTC dataset.

\begin{table}[htbp!]
    \tabcolsep=4pt
    \caption{Mean propensity weighted precision at $\{1,3,5\}$ (\%) of \Algo{Parabel}, \Algo{IPLT}, \Algo{CMT}, \Algo{OPLT} for XMLC tasks. 
    Notation: PP$@k$ -- propensity weighted precision at $k$-position, $R$ -- \Algo{Random} policy, $B$ -- \Algo{Best-greedy} policy.
    }
    \label{tab:xmlc-psp-results}
    \begin{center}
    \begin{tabular}{
        l 
        | r r r
        | r r r
        | r r r 
        | r r r
    }
    \toprule 
    & \multicolumn{3}{c|}{AmazonCat-13K}
    & \multicolumn{3}{c|}{Wiki10-31K}
    & \multicolumn{3}{c|}{WikiLSHTC-320K} 
    & \multicolumn{3}{c}{Amazon-670K} 
    \\
    
    Algorithm
    & \multicolumn{1}{c}{PP$@1$} & \multicolumn{1}{c}{PP$@3$} & \multicolumn{1}{c|}{PP$@5$} 
    & \multicolumn{1}{c}{PP$@1$} & \multicolumn{1}{c}{PP$@3$} & \multicolumn{1}{c|}{PP$@5$}
    & \multicolumn{1}{c}{PP$@1$} & \multicolumn{1}{c}{PP$@3$} & \multicolumn{1}{c|}{PP$@5$}
    & \multicolumn{1}{c}{PP$@1$} & \multicolumn{1}{c}{PP$@3$} & \multicolumn{1}{c}{PP$@5$}
    \\
    \midrule
    
    \Algo{Parabel} & 50.89 & 63.60 & 71.27 & 11.73 & 12.70 & 13.68 & 25.91 & 31.50 & 34.77 & 25.39 & 28.57 & 31.21 \\
    \Algo{IPLT} & 49.97 & 63.13 & 70.77 & 12.26 & 13.75 & 14.80 & 24.16 & 29.09 & 32.29 & 25.25 & 28.85 & 32.12 \\
    
    \midrule
    
    \Algo{CMT} & 47.48 & 54.25 & 56.03 & 9.68 & 9.66 & 9.67 & - & - & - & - & - & - \\
    
    \Algo{OPLT}$_R$ & 48.98 & 60.67 & 67.51 & 11.64 & 13.07 & 14.12 & 16.12 & 19.70 & 22.68 & 20.28 & 24.04 & 27.35 \\
    \Algo{OPLT}$_B$ & 49.30 & 61.55 & 68.64 & 11.59 & 13.19 & 14.26 & 20.69 & 24.87 & 28.14 & 23.41 & 27.26 & 30.49 \\
    \Algo{{OPLT-W}} & 49.86 & 62.91 & 70.44 & 11.69 & 13.41 & 14.43 & 22.56 & 27.10 & 30.23 & 23.65 & 27.62 & 30.95 \\

    \bottomrule
    \end{tabular}
    \end{center}
\end{table}

In Table~\ref{tab:xmlc-results-details}, we show detailed results of the empirical study we performed to evaluate \Algo{OPLT}s comprehensively. In addition to precision at $\{1,3,5\}$ and train times, we also report test times. We present the results for every algorithm that is using an incremental learning algorithm, updating the tree nodes (\Algo{IPLT}, \Algo{CMT}, and \Algo{OPLT}) after 1 and 3 passes over the training dataset. We also present the results of \Algo{OPLT-W} with a warm-start sample of size 5, 10, and 15\% of the training dataset (warm-up dataset).
In all cases, the initial tree was created using hierarchical 2-means clustering on the warm-up sample and later extended using \Algo{Best-greedy} policy. The results show that \Algo{IPLT} and \Algo{OPLT} achieve good results after just one pass over datasets. Prediction times of \Algo{OPLT} are slightly worse than \Algo{IPLT}, probably due to the worst tree structure,  the prediction times of \Algo{OPLT} with warm-start are better than \Algo{OPLT} build from scratch.  The reported prediction times of \Algo{Parabel} are given for the batch prediction. The prediction times seem to be faster, but \Algo{Parabel} needs to decompress the model during prediction, which makes it less suitable for online prediction. It is only efficient when the batches are sufficiently large.

\begin{table}[htbp!]
    \tabcolsep=4pt
    \caption{Mean precision at $\{1,3,5\}$ (\%), train and test CPU time (averaged over 5 runs) of \Algo{Parabel}, \Algo{IPLT}, \Algo{CMT}, \Algo{OPLT} for XMLC tasks.
    Notation: P$@k$ -- precision at $k$-position, $T$ -- CPU time, $R$ -- \Algo{Random} policy, $B$ -- \Algo{Best-greedy} policy, $w$ -- percentage of the training dataset sampled for warm-start, $p$ -- number of passes over train dataset, $*$ -- offline prediction, depends on the test dataset size. Highlighted values are presented in the main paper in Table~\ref{tab:xmlc-precision-results}.
    }
    \label{tab:xmlc-results-details}
    \begin{center}
    \begin{tabular}{
        l 
        | r r r r r
        | r r r r r
    }
    
    \toprule 
    & \multicolumn{5}{c|}{AmazonCat}
    & \multicolumn{5}{c}{Wiki10}
    \\
    Algorithm
    & \multicolumn{1}{c}{P$@1$} & \multicolumn{1}{c}{P$@3$} & \multicolumn{1}{c}{P$@5$} & \multicolumn{1}{c}{$t_{train}$} & \multicolumn{1}{c|}{$t_{test}/N$}
    & \multicolumn{1}{c}{P$@1$} & \multicolumn{1}{c}{P$@3$} & \multicolumn{1}{c}{P$@5$} & \multicolumn{1}{c}{$t_{train}$} & \multicolumn{1}{c}{$t_{test}/N$}
    \\
    \midrule 
    \Algo{Parabel} & \colorbox{blue!30}{92.58} & \colorbox{blue!30}{78.53} & \colorbox{blue!30}{63.90} & \colorbox{blue!30}{11.51m} & $^*$0.3ms & \colorbox{blue!30}{84.17} & \colorbox{blue!30}{72.12} & \colorbox{blue!30}{63.30} & \colorbox{blue!30}{4.00m} & $^*$0.9ms \\
    \Algo{IPLT}$_{p=1}$ & \colorbox{blue!30}{93.11} & \colorbox{blue!30}{78.72} & \colorbox{blue!30}{63.98} & \colorbox{blue!30}{34.2m} & 0.6ms & \colorbox{blue!30}{84.87} & \colorbox{blue!30}{74.42} & \colorbox{blue!30}{65.31} & \colorbox{blue!30}{18.3m} & 10.1ms \\
    \Algo{IPLT}$_{p=3}$ & 93.19 & 78.64 & 63.92 & 115.7m & 0.6ms & 84.37 & 73.90 & 64.70 & 46.8m & 10.4ms \\
    \midrule
    
    
    \Algo{CMT}$_{p=1}$ & 87.51 & 69.49 & 53.99 & 45.8m & 1.1ms & \colorbox{blue!30}{80.59} & \colorbox{blue!30}{64.17} & \colorbox{blue!30}{55.05} & \colorbox{blue!30}{35.1m} & 16.5ms \\
    \Algo{CMT}$_{p=3}$ & \colorbox{blue!30}{89.43} & \colorbox{blue!30}{70.49} & \colorbox{blue!30}{54.23} & \colorbox{blue!30}{168.3m} & 1.5ms & 79.86 & 64.72 & 55.25 & 120.9m & 20.4ms \\
    
    \Algo{OPLT}$_{R,p=1}$ & \colorbox{blue!30}{92.66} & \colorbox{blue!30}{77.44} & \colorbox{blue!30}{62.52} & \colorbox{blue!30}{99.5m} & 1.7ms & \colorbox{blue!30}{84.34} & \colorbox{blue!30}{73.73} & \colorbox{blue!30}{64.41} & \colorbox{blue!30}{30.3m} & 14.1ms \\
    \Algo{OPLT}$_{R,p=3}$ & 92.65 & 77.43 & 62.59 & 278.3m & 1.4ms & 83.09 & 72.25 & 63.33 & 86.5m & 18.2ms \\
    \Algo{OPLT}$_{B,p=1}$ & \colorbox{blue!30}{92.74} & \colorbox{blue!30}{77.74} & \colorbox{blue!30}{62.91} & \colorbox{blue!30}{84.1m} & 1.5ms & \colorbox{blue!30}{84.47} & \colorbox{blue!30}{73.73} & \colorbox{blue!30}{64.39} & \colorbox{blue!30}{27.7m} & 16.4ms \\
    \Algo{OPLT}$_{B,p=3}$ & 92.77 & 77.70 & 62.93 & 251.9m & 1.4ms & 83.39 & 72.50 & 63.48 & 86.7m & 18.2ms \\
    \midrule
    \Algo{OPLT-W}$_{w=5\%,p=1}$ & 93.14 & 78.67 & 63.85 & 40.6m & 0.8ms & 85.10 & 74.44 & 64.77 & 28.7m & 16.6ms \\
    \Algo{OPLT-W}$_{w=5\%,p=3}$ & 93.11 & 78.50 & 63.75 & 134.3m & 0.8ms & 83.93 & 73.34 & 63.75 & 75.9m & 17.0ms \\
    \Algo{OPLT-W}$_{w=10\%,p=1}$ & 93.14 & 78.68 & 63.92 & 43.7m & 0.9ms & 85.22 & 74.68 & 64.93 & 28.2m & 17.8ms \\
    \Algo{OPLT-W}$_{w=10\%,p=3}$ & \colorbox{blue!30}{93.10} & \colorbox{blue!30}{78.51} & \colorbox{blue!30}{63.79} & \colorbox{blue!30}{133.9m} & 0.8ms & \colorbox{blue!30}{84.75} & \colorbox{blue!30}{74.02} & \colorbox{blue!30}{64.33} & \colorbox{blue!30}{83.9m} & 18.3ms \\
    \Algo{OPLT-W}$_{w=15\%,p=1}$ & 93.17 & 78.75 & 63.95 & 50.1m & 1.0ms & 85.42 & 74.50 & 64.94 & 28.2m & 17.5ms \\
    \Algo{OPLT-W}$_{w=15\%,p=3}$ & 93.15 & 78.54 & 63.82 & 121.6m & 0.8ms & 84.69 & 73.90 & 64.29 & 79.1m & 19.5ms \\
    
    \toprule 
    & \multicolumn{5}{c|}{WikiLSHTC}
    & \multicolumn{5}{c}{Amazon}
    \\
    
    Algorithm
    & \multicolumn{1}{c}{P$@1$} & \multicolumn{1}{c}{P$@3$} & \multicolumn{1}{c}{P$@5$} & \multicolumn{1}{c}{$t_{train}$} & \multicolumn{1}{c|}{$t_{test}/N$}
    & \multicolumn{1}{c}{P$@1$} & \multicolumn{1}{c}{P$@3$} & \multicolumn{1}{c}{P$@5$} & \multicolumn{1}{c}{$t_{train}$} & \multicolumn{1}{c}{$t_{test}/N$}
    \\
    \midrule 
    \Algo{Parabel} & \colorbox{blue!30}{62.78} & \colorbox{blue!30}{41.22} & \colorbox{blue!30}{30.27} & \colorbox{blue!30}{15.4m} & $^*$0.3ms & \colorbox{blue!30}{43.13} & \colorbox{blue!30}{37.94} & \colorbox{blue!30}{34.00} & \colorbox{blue!30}{7.6m} & $^*$0.3ms \\
    \Algo{IPLT}$_{p=1}$ & 58.14 & 37.94 & 28.14 & 69.0m & 2.1ms & 40.78 & 35.88 & 32.28 & 33.2m & 5.1ms \\
    \Algo{IPLT}$_{p=3}$ & \colorbox{blue!30}{60.80} & \colorbox{blue!30}{39.58} & \colorbox{blue!30}{29.24} & \colorbox{blue!30}{175.1m} & 2.6ms & \colorbox{blue!30}{43.55} & \colorbox{blue!30}{38.69} & \colorbox{blue!30}{35.20} & \colorbox{blue!30}{79.4m} & 6.2ms \\
    \midrule
    \Algo{OPLT}$_{R,p=1}$ & 46.36 & 29.85 & 22.53 & 103.0m & 6.4ms & 36.27 & 32.13 & 29.01 & 46.1m & 17.0ms \\
    \Algo{OPLT}$_{R,p=3}$ & \colorbox{blue!30}{47.76} & \colorbox{blue!30}{30.97} & \colorbox{blue!30}{23.37} & \colorbox{blue!30}{330.1m} & 7.9ms & \colorbox{blue!30}{38.42} & \colorbox{blue!30}{34.33} & \colorbox{blue!30}{31.32} & \colorbox{blue!30}{134.2m} & 18.2ms \\
    \Algo{OPLT}$_{B,p=1}$ & 53.32 & 34.22 & 25.52 & 88.6m & 4.8ms & 38.42 & 34.03 & 30.73 & 36.7m & 10.0ms \\
    \Algo{OPLT}$_{B,p=3}$ & \colorbox{blue!30}{54.69} & \colorbox{blue!30}{35.32} & \colorbox{blue!30}{26.31} & \colorbox{blue!30}{300.0m} & 5.8ms & \colorbox{blue!30}{41.09} & \colorbox{blue!30}{36.65} & \colorbox{blue!30}{33.42} & \colorbox{blue!30}{111.9m} & 13.0ms \\
    \midrule
    \Algo{OPLT-W}$_{w=5\%,p=1}$ & 57.46 & 36.95 & 27.35 & 82.5m & 3.6ms & 39.45 & 34.85 & 31.44 & 35.3m & 9.1ms \\
    \Algo{OPLT-W}$_{w=5\%,p=3}$ & 58.51 & 37.92 & 28.04 & 249.8m & 4.6ms & 41.96 & 37.39 & 34.07 & 100.0m & 11.1ms \\
    \Algo{OPLT-W}$_{w=10\%,p=1}$ & 58.20 & 37.50 & 27.73 & 71.9m & 3.0ms & 39.71 & 35.03 & 31.57 & 36.6m & 9.1ms \\
    \Algo{OPLT-W}$_{w=10\%,p=3}$ & \colorbox{blue!30}{59.23} & \colorbox{blue!30}{38.39} & \colorbox{blue!30}{28.38} & \colorbox{blue!30}{205.7m} & 3.4ms & \colorbox{blue!30}{42.21} & \colorbox{blue!30}{37.60} & \colorbox{blue!30}{34.25} & \colorbox{blue!30}{98.3m} & 11.3ms \\
    \Algo{OPLT-W}$_{w=15\%,p=1}$ & 58.68 & 37.85 & 27.97 & 72.1m & 2.8ms & 40.11 & 35.35 & 31.84 & 34.0m & 8.3ms \\
    \Algo{OPLT-W}$_{w=15\%,p=3}$ & 59.66 & 38.67 & 28.57 & 196.0m & 3.1ms & 42.41 & 37.70 & 34.29 & 85.1m & 9.1ms \\

    \bottomrule
    \end{tabular}
    \end{center}
\end{table}

\section{Detailed results of \Algo{OPLT} and \Algo{CMT} on few-shot multi-class classification tasks}
\label{app:detail-results-fs}

In Table~\ref{tab:fs-results-detailes} we present additional results for experiments on few-shot multi-class classification. We report results on the test set after 1 and 3 passes. In addition to accuracy and train times, we also report test times after 1 and 3 passes over the training dataset.

\begin{table}[htbp!]
    \tabcolsep=3pt
    \caption{Mean accuracy of prediction (\%) and train and test CPU time of \Algo{CMT}, \Algo{OPLT} for few-shot multi-class classification tasks.
    Notation: Acc -- accuracy, $T$ -- CPU time, $N$ -- number of samples in test set, $R$ -- \Algo{Random} policy, $B$ -- \Algo{Best-greedy} policy, $p$ -- number of passes over train dataset. Highlighted values are presented in the main paper in Table~\ref{tab:fs-results}.}
    \label{tab:fs-results-detailes}
    \begin{center}
    \begin{tabular}{
        l 
        | r r r
        | r r r 
        | r r r
        | r r r 
    }
    \toprule 
    & \multicolumn{3}{c|}{ALOI} 
    & \multicolumn{3}{c|}{Wikipara 1-shot}
    & \multicolumn{3}{c|}{Wikipara 3-shot}
    & \multicolumn{3}{c}{Wikipara 5-shot}
    \\
    
    Algorithm 
    & \multicolumn{1}{c}{Acc.} & \multicolumn{1}{c}{$t_{train}$} & \multicolumn{1}{c|}{$t_{test}/N$}
    & \multicolumn{1}{c}{Acc.} & \multicolumn{1}{c}{$t_{train}$} & \multicolumn{1}{c|}{$t_{test}/N$}
    & \multicolumn{1}{c}{Acc.} & \multicolumn{1}{c}{$t_{train}$} & \multicolumn{1}{c|}{$t_{test}/N$}
    & \multicolumn{1}{c}{Acc.} & \multicolumn{1}{c}{$t_{train}$} & \multicolumn{1}{c}{$t_{test}/N$}
    \\
    \midrule

    \Algo{CMT}$_{p=1}$ & 17.63 & 37.8s & 0.78ms & 2.22 & 7.1s & 0.51ms & 3.22 & 20.9s & 0.76ms & 4.15 & 81.1s & 3.22ms \\
    \Algo{OPLT}$_{R,p=1}$ & 62.58 & 6.0s & 0.12ms & 1.51 & 2.6s & 0.82ms & 13.31 & 6.7s & 1.83ms & 26.06 & 11.6s & 2.61ms \\
    \Algo{OPLT}$_{B,p=1}$ & 63.91 & 6.1s & 0.12ms & 0.80 & 2.2s & 0.55ms & 11.40 & 6.4s & 1.67ms & 23.12 & 10.3s & 2.48ms \\
    
    \midrule
    \Algo{CMT}$_{p=3}$ & \colorbox{blue!30}{71.98} & \colorbox{blue!30}{207s} & 0.57ms & 2.48 & 21.3s & 0.37ms & \colorbox{blue!30}{3.28} & \colorbox{blue!30}{63.1s} & 0.49ms & \colorbox{blue!30}{4.21} & \colorbox{blue!30}{240.9s} & 1.22ms \\
    \Algo{OPLT}$_{R,p=3}$ & \colorbox{blue!30}{66.50} & \colorbox{blue!30}{29.3s} & 0.11ms & 8.99 & 4.6s & 1.14ms & \colorbox{blue!30}{27.34} & \colorbox{blue!30}{16.4s} & 2.64ms & \colorbox{blue!30}{40.67} & \colorbox{blue!30}{27.6s} & 2.83ms \\
    \Algo{OPLT}$_{B,p=3}$ & \colorbox{blue!30}{67.26} & \colorbox{blue!30}{18.1s} & 0.10ms & 3.31 & 4.4s & 0.87ms & \colorbox{blue!30}{24.66} & \colorbox{blue!30}{15.6s} & 2.37ms & \colorbox{blue!30}{39.13} & \colorbox{blue!30}{27.2s} & 2.54ms \\

    \bottomrule
    \end{tabular}
    \end{center}
\end{table}

\section{Performance of \Algo{OPLT} with \Algo{Best-greedy} policy and different $\alpha$ values}
\label{app:tree-alpha}

In Table~\ref{tab:xmlc-tree-alpha}, \Algo{OPLT} with \Algo{Best-greedy} policy and different values of tree balancing parameter $\alpha$ is compared to \Algo{IPLT} and \Algo{OPLT} with \Algo{Random} policy.
All other parameters of the model were set as described in Section~\ref{sec:xmlc-experiments}.
It shows that for $\alpha \ge 0.75$, \Algo{OPLT} achieves the tree depth close or the same as perfectly balanced tree build for \Algo{IPLT}, at the same time having the best predictive performance and the shortest training and prediction time among \Algo{OPLT} variants.

\begin{table}[htbp!]
    \tabcolsep=3pt
    \caption{Mean precision at $\{1,3,5\}$ (\%), train time, test time and tree depth (averaged over 5 runs) of \Algo{OPLT} with \Algo{Best-greedy} policy with different $\alpha$ values compared with \Algo{IPLT} and \Algo{OPLT} with \Algo{Random} policy. 
    Notation: P$@k$ -- precision at $k$-position, $t$ -- CPU time, d($T$) -- tree depth, $R$ -- \Algo{Random} policy, $B$ -- \Algo{Best-greedy} policy, $\alpha$ -- tree balancing parameter, percentage of the training dataset sampled for warm-start, $p$ -- number of passes over train dataset.}
    \label{tab:xmlc-tree-alpha}
    \begin{center}
    \begin{tabular}{
        l 
        | r r r r r r
        | r r r r r r
    }

    \toprule 
    & \multicolumn{6}{c|}{AmazonCat}
    & \multicolumn{6}{c}{Wiki10}
    \\
    Algorithm
    & \multicolumn{1}{c}{P$@1$} & \multicolumn{1}{c}{P$@3$} & \multicolumn{1}{c}{P$@5$} & \multicolumn{1}{c}{$t_{train}$} & \multicolumn{1}{c}{$t_{test}/N$} & \multicolumn{1}{c|}{d($T$)}
    & \multicolumn{1}{c}{P$@1$} & \multicolumn{1}{c}{P$@3$} & \multicolumn{1}{c}{P$@5$} & \multicolumn{1}{c}{$t_{train}$} & \multicolumn{1}{c}{$t_{test}/N$} & \multicolumn{1}{c}{d($T$)}\\
    \midrule 
    \Algo{IPLT}$_{p=3}$ & 93.10 & 78.52 & 63.81 & 115.7m & 0.6ms & 10.0 & 83.49 & 73.06 & 64.12 & 46.8m & 10.4ms & 11.0 \\
    \Algo{OPLT}$_{R,p=3}$ & 92.65 & 77.43 & 62.59 & 278.3m & 1.4ms & 10.0 & 83.09 & 72.25 & 63.33 & 86.5m & 18.2ms & 11.0 \\
    \midrule
    \Algo{OPLT}$_{B,\alpha=0.25,p=3}$ & 92.77 & 77.77 & 62.99 & 303.1m & 1.5ms & 30.0 & 82.93 & 72.41 & 63.38 & 88.2m & 20.0ms & 24.6 \\
    \Algo{OPLT}$_{B,\alpha=0.375,p=3}$ & 92.80 & 77.72 & 62.93 & 250.2m & 1.5ms & 18.0 & 83.10 & 72.43 & 63.37 & 84.6m & 18.9ms & 17.2 \\
    \Algo{OPLT}$_{B,\alpha=0.5,p=3}$ & 92.79 & 77.74 & 62.95 & 257.6m & 1.4ms & 14.0 & 83.09 & 72.48 & 63.45 & 86.0m & 18.7ms & 14.0 \\
    \Algo{OPLT}$_{B,\alpha=0.625,p=3}$ & 92.77 & 77.74 & 62.96 & 245.2m & 1.4ms & 12.0 & 83.25 & 72.50 & 63.48 & 87.3m & 19.6ms & 12.6 \\
    \Algo{OPLT}$_{B,\alpha=0.75,p=3}$ & 92.77 & 77.70 & 62.93 & 248.9m & 1.4ms & 11.0 & 83.39 & 72.50 & 63.48 & 86.7m & 18.2ms & 12.0 \\
    \Algo{OPLT}$_{B,\alpha=0.875,p=3}$ & 92.73 & 77.64 & 62.85 & 207.7m & 1.2ms & 10.0 & 83.41 & 72.50 & 63.44 & 82.2m & 17.6ms & 11.0 \\
    
    \toprule 
    & \multicolumn{6}{c|}{WikiLSHTC}
    & \multicolumn{6}{c}{Amazon}
    \\
    
    Algorithm
    & \multicolumn{1}{c}{P$@1$} & \multicolumn{1}{c}{P$@3$} & \multicolumn{1}{c}{P$@5$} & \multicolumn{1}{c}{$t_{train}$} & \multicolumn{1}{c}{$t_{test}/N$} & \multicolumn{1}{c|}{d($T$)}
    & \multicolumn{1}{c}{P$@1$} & \multicolumn{1}{c}{P$@3$} & \multicolumn{1}{c}{P$@5$} & \multicolumn{1}{c}{$t_{train}$} & \multicolumn{1}{c}{$t_{test}/N$} & \multicolumn{1}{c}{d($T$)}
    \\
    \midrule 
    \Algo{IPLT}$_{p=3}$ & 60.80 & 39.58 & 29.24 & 175.1m & 2.6ms & 14.0 & 43.55 & 38.69 & 35.20 & 79.4m & 6.2ms & 15.0 \\
    \Algo{OPLT}$_{R,p=3}$ & 47.76 & 30.97 & 23.37 & 330.1m & 7.9ms & 14.8 & 38.42 & 34.33 & 31.32 & 134.2m & 18.2ms & 16.0 \\
    \midrule
    \Algo{OPLT}$_{B,\alpha=0.25,p=3}$ & 53.14 & 34.36 & 25.69 & 337.8m & 6.3ms & 68.4 & 38.83 & 34.72 & 31.69 & 139.2m & 16.4ms & 38.4 \\
    \Algo{OPLT}$_{B,\alpha=0.375,p=3}$ & 53.47 & 34.57 & 25.81 & 309.1m & 6.9ms & 30.8 & 39.17 & 35.04 & 31.99 & 126.6m & 16.2ms & 21.2 \\
    \Algo{OPLT}$_{B,\alpha=0.5,p=3}$ & 53.77 & 34.74 & 25.93 & 294.3m & 6.0ms & 21.8 & 39.47 & 35.26 & 32.17 & 123.5m & 14.1ms & 17.8 \\
    \Algo{OPLT}$_{B,\alpha=0.625,p=3}$ & 54.01 & 34.90 & 26.04 & 298.7m & 6.1ms & 18.0 & 40.06 & 35.79 & 32.65 & 118.9m & 14.0ms & 16.0 \\
    \Algo{OPLT}$_{B,\alpha=0.75,p=3}$ & 54.69 & 35.32 & 26.31 & 297.0m & 5.8ms & 16.0 & 41.09 & 36.65 & 33.42 & 111.9m & 13.0ms & 16.0 \\
    \Algo{OPLT}$_{B,\alpha=0.875,p=3}$ & 54.56 & 35.22 & 26.25 & 285.4m & 5.8ms & 15.0 & 41.18 & 36.75 & 33.53 & 110.1m & 12.7ms & 15.0 \\

    \bottomrule
    \end{tabular}
    \end{center}
\end{table}

\section{Performance analysis of \Algo{OPLT} with warm-start}
\label{app:warm-start}

In Table~\ref{tab:xmlc-warm-start}, we compare \Algo{OPLT} with warm-start (\Algo{OPLT-W}) with two variants of \Algo{IPLT}. In the first variant, \Algo{IPLT} is trained only on the warm-up training dataset, and in the second variant, \Algo{IPLT} uses a tree created on warm-up, but then updated with all examples from the training set but without updating the tree structure. In both variants, \Algo{IPLT} cannot predict labels that are not present in the initial warm-up dataset.
All other parameters of the model were set as described in Section~\ref{sec:xmlc-experiments}.
This experiment shows the significant gain in predictive performance on WikiLSHTC and Amazon datasets by extending a tree with newly observed labels over the \Algo{IPLT} variants that do not take new labels into account.

\begin{table}[htbp!]
    \tabcolsep=4pt
    \caption{Mean precision at $\{1,3,5\}$ (\%, averaged over 5 runs) of \Algo{IPLT} trained only on warm-start training dataset and \Algo{IPLT} with a tree created on warm-start training dataset (\Algo{IPLT-U}) and updated with all examples without updating the tree and \Algo{OPLT} with warm-start (\Algo{OPLT-W}). 
    Notation: P$@k$ -- precision at $k$-position, $w$ -- percentage of the training dataset sampled for warm-start, $p$ -- number of passes over train dataset.}
    \label{tab:xmlc-warm-start}
    \begin{center}
    \begin{tabular}{
        l 
        | r r r
        | r r r
        | r r r
        | r r r
    }
    \toprule 
    & \multicolumn{3}{c|}{AmazonCat}
    & \multicolumn{3}{c|}{Wiki10}
    & \multicolumn{3}{c|}{WikiLSHTC} 
    & \multicolumn{3}{c}{Amazon} 
    \\
    
    Algorithm 
    & \multicolumn{1}{c}{P$@1$} & \multicolumn{1}{c}{P$@3$} & \multicolumn{1}{c|}{P$@5$} 
    & \multicolumn{1}{c}{P$@1$} & \multicolumn{1}{c}{P$@3$} & \multicolumn{1}{c|}{P$@5$}
    & \multicolumn{1}{c}{P$@1$} & \multicolumn{1}{c}{P$@3$} & \multicolumn{1}{c|}{P$@5$}
    & \multicolumn{1}{c}{P$@1$} & \multicolumn{1}{c}{P$@3$} & \multicolumn{1}{c}{P$@5$}
    \\
    \midrule
    
    \Algo{IPLT}$_{w=5\%,p=3}$ & 88.49 & 70.48 & 55.49 & 80.69 & 61.68 & 51.85 & 38.95 & 22.74 & 16.37 & 10.98 & 9.55 & 8.77 \\
    \Algo{IPLT-U}$_{w=5\%,p=3}$ & 93.15 & 78.54 & 63.62 & 83.95 & 72.99 & 63.20 & 54.70 & 34.54 & 25.27 & 29.54 & 22.65 & 18.22 \\
    \Algo{OPLT-W}$_{w=5\%,p=3}$ & 93.11 & 78.50 & 63.75 & 83.93 & 73.34 & 63.75 & 58.51 & 37.92 & 28.04 & 41.96 & 37.39 & 34.07 \\
    \midrule
    \Algo{IPLT}$_{w=10\%,p=3}$ & 89.92 & 72.97 & 57.99 & 80.79 & 65.46 & 55.38 & 44.59 & 26.71 & 19.35 & 15.77 & 13.70 & 12.45 \\
    \Algo{IPLT-U}$_{w=10\%,p=3}$ & 93.12 & 78.58 & 63.79 & 84.38 & 73.52 & 63.88 & 57.47 & 36.81 & 27.05 & 35.41 & 28.84 & 24.06 \\
    \Algo{OPLT-W}$_{w=10\%,p=3}$ & 93.10 & 78.51 & 63.79 & 84.75 & 74.02 & 64.33 & 59.23 & 38.39 & 28.38 & 42.21 & 37.60 & 34.25 \\
    \midrule
    \Algo{IPLT}$_{w=15\%,p=3}$ & 90.69 & 74.17 & 59.25 & 81.61 & 67.28 & 57.18 & 47.63 & 29.03 & 21.11 & 19.38 & 16.79 & 15.24 \\
    \Algo{IPLT-U}$_{w=15\%,p=3}$ & 93.13 & 78.62 & 63.85 & 84.65 & 73.93 & 64.24 & 58.55 & 37.71 & 27.77 & 37.90 & 31.87 & 27.19 \\
    \Algo{OPLT-W}$_{w=15\%,p=3}$ & 93.15 & 78.54 & 63.82 & 84.69 & 73.90 & 64.29 & 59.66 & 38.67 & 28.57 & 42.41 & 37.70 & 34.29 \\

    \bottomrule
    \end{tabular}
    \end{center}
\end{table}

\end{document}